\documentclass[11pt,letterpaper]{article}

\usepackage[top=1in, bottom=1in, left=1in, right=1in]{geometry}

\usepackage{color}
\usepackage{bm}
\usepackage{amsmath}
\usepackage{amssymb}
\usepackage{caption}
\usepackage{subcaption}
\usepackage{graphicx}
\usepackage{booktabs}
\usepackage{lipsum}
\usepackage[table]{xcolor}
\usepackage{tikz}
\usepackage{url}
\usepackage{amsthm}

\newcommand{\bS}{\mathbf{S}}
\newcommand{\bpi}{\bm{\pi}}
\newcommand{\bc}{\mathbf{c}}

\newcommand{\real}{\mathbb{R}}
\newcommand{\bQ}{\textbf{Q}}

\newcommand{\bw}{\textbf{w}} 
 
\newcommand{\bx}{\textbf{x}}
\newcommand{\bq}{\textbf{q}}

\newcommand{\bzero}{\mathbf{0}} 
\newcommand{\bH}{\textbf{H}} 
 
\newcommand{\bM}{\textbf{M}}
\newcommand{\bX}{\textbf{X}} 
\newcommand{\bW}{\textbf{W}} 
\newcommand{\bZ}{\textbf{Z}} 
\newcommand{\bWhat}{\hat{\textbf{W}}} 
\newcommand{\bE}{\textbf{E}} 
\newcommand{\be}{\textbf{e}} 
\newcommand{\bA}{\textbf{A}} 
\newcommand{\bB}{\textbf{B}} 
\newcommand{\bY}{\textbf{Y}} 

\newcommand{\lambdahat}{\hat\lambda}

\DeclareMathOperator{\Loss}{L}
\newcommand{\bXr}{\bX^{r}}
\newcommand{\bXnr}{\bX^{\bar{r}}}
\newcommand{\T}{\intercal}
\newcommand{\E}{\mathbb{E}}
\newcommand{\prob}{\mathbb{P}}
\newcommand{\eval}{\Lambda}
\newcommand{\bHhat}{\hat{\bH}}
\newcommand{\nbr}{\bS_r}
 
\newcommand{\calC}{\mathcal{C}}

\newcommand{\calZ}{\mathcal{Z}}
\newcommand{\calF}{\mathcal{F}}
\newcommand{\encode}{\mathcal{E}}

\newcommand{\calB}{\mathcal{B}}
\newcommand{\tB}{\text{B}}
\DeclareMathOperator{\f}{vec}
\DeclareMathOperator{\Prev}{prev}

\newtheorem{theorem}{Theorem} 
 
\newtheorem{lemma}{Lemma}
\newtheorem{assumption}{Assumption} 

\newtheorem{definition}{Definition}

\title{Learning discrete Bayesian networks in polynomial time and sample complexity}

\author{{\bf Adarsh Barik} \\
	Department of Computer Science \\
	Purdue University \\
	\and
	{\bf Jean Honorio} \\
	Department of Computer Science \\
	Purdue University \\  }

\date{}

\begin{document}

\maketitle

\begin{abstract}
  In this paper, we study the problem of structure learning for Bayesian networks in which nodes take discrete values. The problem is NP-hard in general but we show that under certain conditions we can recover the true structure of a Bayesian network with sufficient number of samples. We develop a mathematical model which does not assume any specific conditional probability distributions for the nodes. We use a primal-dual witness construction to prove that, under some technical conditions on the interaction between node pairs, we can do exact recovery of the parents and children of a node by performing group $\ell_{12}$-regularized multivariate regression. Thus, we recover the true Bayesian network structure. If degree of a node is bounded then the sample complexity of our proposed approach grows logarithmically with respect to the number of nodes in the Bayesian network. Furthermore, our method runs in polynomial time.
\end{abstract}

\section{Introduction} 
\label{sec:introduction}

\paragraph{Motivation.} Probabilistic graphical models provide a framework to model complex systems. They use graphs to represent variables along with their conditional dependencies and enable us to formally understand the interaction among different variables. Based on the type of the graph and modeling of the conditional dependencies, there are various classes of graphical models. One of the important classes are \emph{Bayesian networks} \cite{koller2009probabilistic, lauritzen1996graphical} which use a directed acyclic graph (DAG) to encode relationships among the variables. The variables are represented as nodes of the DAG and a directed edge from node $i$ to node $j$ denotes that node $i$ is a parent of node $j$. 

\begin{definition}[Bayesian Network]
	A Bayesian network  on $n$ random variables $\mathbb{X} = \{X_1, X_2,\dots,X_n\}$ is a DAG that specifies a joint distribution over $\mathbb{X}$ as a product of conditional probability functions $\prob(X_r | X_{\bpi_r})$, one for each variable $X_r$ given its set of parents $\bpi_r \subseteq \{ 1,\dots, n \}$. The joint probability
distribution over all nodes is given by:
\begin{align}
\label{eq:bayesian_network}
\prob(\mathbb{X}) = \prod_{r=1}^n \prob(X_r | X_{\bpi_r}) 
\end{align}
where $(\forall r, X_{\bpi_r}) {\rm\ } \sum_{X_r} \prob(X_r | X_{\bpi_r}) = 1$ and therefore
$\prob(\mathbb{X})$ is valid, i.e., $\sum_{\mathbb{X}} \prob(\mathbb{X}) = 1$.
\end{definition}

It is quite common to see categorical variables in the real systems. For example, the country of residence of a person may take values from a set \{United States, China, England, \ldots \}. Consequently, categorical random variables frequently appear in Bayesian networks. Since these variables are not ordinal, it becomes important that we do not introduce any artificial ordering while using them in our mathematical model. In this paper, we propose a method to learn the skeleton of a Bayesian network where all the nodes are categorical random variables. 
\paragraph{Related work.} The structure of a Bayesian network provides great insights into the complex interactions amongst variables. Thus, a considerable amount of work has been done in this field and several different methods have been proposed to learn Bayesian networks from data. We can broadly divide these methods in two categories. First, there are methods which learn the DAG from data by maximizing a well defined score. In this category, there are some heuristics based approaches such as \cite{friedman1999learning,tsamardinos2006max, margaritis2000bayesian, moore2003optimal}. There are also some exact but exponential-time score maximizing algorithms such as \cite{koivisto2004exact, silander2012simple, cussens2012bayesian, cussens2012bayesiancut,jaakkola2010learning}. Second, there are independence test based methods which determine the edge between two nodes by conducting dependence or independence tests.  For example, \cite{spirtes2000causation, cheng2002learning, yehezkel2005recursive, xie2008recursive} use this approach. 
There are also some results available for special classes of Bayesian networks. Ghoshal et al.~\cite{ghoshal2017learning} provide polynomial sample and time complexities guarantees for structure learning in Gaussian Bayesian networks. A more general result is also provided for linear structural equation models in~\cite{ghoshal2017SEMlearning}. For discrete variables, Park et al.~\cite{park2015learning} provide statistical guarantees for recovery of the node ordering in polynomial time and sample complexities for Poisson-distributed variables. More general results are also provided for other ordinal variables with binomial, geometric, exponential and gamma distributions in \cite{park2017learning}. However, as stated by the authors, their method does not work for Bernoulli or multinomial distributions.  Brenner et al.~\cite{brenner2013sparsityboost} proposed a method which works with binary variables exclusively. Their proposed method has a sample complexity  of the order $O(n^2)$, where $n$ is number of nodes in the Bayesian network. Note although that the method of \cite{brenner2013sparsityboost} is worst-case exponential time.

\paragraph{Learning Bayesian networks is hard.} The problem of learning the structure of a Bayesian network from data is amongst the hardest problems to be solved, from the computational viewpoint. Independence test methods require a number of tests that grows exponentially in the number of nodes, in the worst case. It is also known that finding the structure of a Bayesian network by score maximization techniques is NP-hard~\cite{chickering2004large}. Thus, unless the long standing problem of P vs. NP is resolved, the problem remains intractable in its general form. This implies that we need to work within the limits of some technical assumptions to solve the problem of structure learning of a Bayesian network with provable computational and statistical efficiency guarantees. 

\paragraph{Contributions.} 
In summary, we make the following contributions in this paper:
	\begin{enumerate}
		\item We formulate the structure recovery problem as a block $\ell_{12}$-regularized multivariate regression problem. We do not assume that the categorical variables are ordinal. Our formulation is also independent of any specific conditional probability distribution for the nodes.
		\item We obtain sufficient conditions for DAG recovery by controlling the interaction between node pairs for arbitrary conditional probability distributions.
		\item We show that if our DAG recovery conditions are satisfied then the sample complexity of our method is logarithmic with respect to the number of nodes. Since our method uses the interior point algorithm, it also runs in polynomial time.  
	\end{enumerate}

\section{Preliminaries}
\label{subsec:notation}

In this section, we introduce formal definitions and notations. We define a Bayesian network on a DAG $G(V, E)$ where $V = \{X_1,\dots, X_n\}$ is the set of $n$ categorical random variables and $E$ is the set of directed edges between them. The set of all the random variables except $X_r$ is denoted by the shorthand notation $X_{\bar r}$. To use categorical variables in our mathematical model, we need to represent them quantitatively~\cite{cohen2013applied}. We do this by encoding them as numerical vectors. Each categorical variable $X_r$ takes values from a set $\calC_r$ with cardinality $m_r$.  The indexing set $\{1,\dots,p\}$ is denoted by $[p]$. For an indexing set $A$, we define $\rho_A \triangleq \sum_{i \in A} (m_i - 1)$. For brevity, a singleton indexing set $\{q\}$ is denoted as $q$ when its use is clear from the context. We define $\bar{\rho} \triangleq \max_{i \in [n] } (m_i - 1 )$. We define our encoder $\encode$ as a map from $\calC_r$ to $\calB^{\rho_r}$ for a bounded and countable set $\calB \subset \real$. In our proofs, we take $\calB = \{-1, 0, 1\}$ which includes commonly used encoding schemes such as dummy encoding and unweighted effects encoding. We denote the encoding of $X_r \in \calC_r$ as $\encode(X_r) \in \calB^{\rho_r}$. By abuse of notation, $\encode(X_{\bar r}) \in \calB^{\rho_{[n]\backslash r}}$ denotes a vector which contains encoding for the set of all the random variables except $X_r$. In the DAG $G(V, E)$, we define the parents set $\bpi_r$ and children set $\bc_r$ for a node $r$ as $\bpi_r = \{i \ |(X_i, X_r) \in E \} \}$ and $\bc_r = \{i \ | (X_r, X_i) \in E \}$. All the other nodes excluding $\bpi_r, \bc_r$ and the node $r$ itself are denoted as $(\bpi_r \cup \bc_r)^c$. The conditional probability distribution for each node $r$ given its parents set $\bpi_r$ is denoted by $\prob(X_r | X_{\bpi_r})$. Data for each node is generated following this conditional distribution. We observe $N$ i.i.d. samples of $n$ nodes of $G(V,E)$. We collect these samples in an $N \times n$ sample matrix $\bX$ where each row represents a sample. The vector $\bXr$ denotes column $r$ of $\bX$ which collects all the samples of $X_r$. Similarly, $\bXnr$ is a matrix containing all the columns except column $r$ of $\bX$. Since these samples contain categorical values, we encode them before we can use them in our mathematical model. Using our encoding scheme, we get encoded sample matrices $\encode(\bX^r) \in \calB^{N \times \rho_r}$ and $\encode(\bX^{\bar r}) \in \calB^{N \times \rho_{[n]\backslash r}}$ corresponding to $\bX^r$ and $\bX^{\bar r}$ respectively. For a matrix $\bA \in \real^{p \times q}$ and two sets $S \subseteq [p]$ and $T \subseteq [q]$,  $\bA_{S T}$ denotes $\bA$ restricted to rows in $S$ and columns in $T$.  Similarly, $\bA_{S.}$ and $\bA_{.T}$ are row and column restricted matrices respectively. We use an operator ``$\f$'' which transforms a matrix $\bA \in \real^{p \times q}$ into a vector $\f(\bA) \in \real^{pq\times 1}$ by stacking the columns of the matrix $\bA$ on top of one another. We use the following vector and matrix norms in our theoretical discussion:

\paragraph{Vector norm} For a vector $\textbf{m} \in \real^q$, the $\ell_p-$norm is defined as $\| \textbf{m} \|_p \triangleq (\sum_{i=1}^q | \textbf{m}_i |^p)^{\frac{1}{p}} $. The $\ell_{\infty}$-norm is defined as $ \| \textbf{m} \|_{\infty} = \max_{i \in [p]} | \textbf{m}_i | $.
\paragraph{Matrix norms}  The Frobenius norm for a matrix $\bA \in \real^{p \times q}$  is defined as $  \| \bA \|_F =  \sqrt{\sum_{i=1}^p \sum_{j=1}^q| \bA_{ij} |^2}$. We define the $(a,b)-$operator norm~\cite{wainwright2015high} for $\bA$ as $
||| \bA |||_{a, b} \triangleq \sup_{\|\bx\|_b = 1} \|\bA \bx\|_a$. Using the above definition, the $\ell_{\infty}$-operator norm for $\bA$ is defined as $ ||| \bA |||_{\infty, \infty} = \max_{i \in [p]}\sum_{j=1}^q | \bA_{ij} | $. Similarly, the spectral norm of $\bA$ is defined as $ ||| \bA |||_{2,2} = \sup_{ \|\bx\|_2 = 1} \|\bA \bx\|_2 $ . 

We also define a block matrix norm for row partitioned block matrices. Let $\bA \in \real^{\sum_{i=1}^k p_i \times q}, \forall i \in [k]$ be a row partitioned block matrix defined as follows: 
\[ 
\bA = \begin{bmatrix}
\bA_1 \\ \vdots \\ \bA_k
\end{bmatrix} \quad \text{where each $\bA_i \in \real^{p_i \times q}$.}
\]
Then $ \| \bA \|_{ \text{B}, a, b} \triangleq \big( \sum_{i=1}^k (\| \f(\bA_i) \|_b)^a \big)^{\frac{1}{a}}$ where $\f(\bA_i)$ flattens the matrix $\bA_i \in \real^{p_i \times q}$ into a vector of size $p_i q$ and $\text{B}$ indicates that we are dealing with a block norm. For example, $\| \bA \|_{\text{B}, \infty, 2} = \max_{i \in [k]} \| \f(\bA_i) \|_2 $, $ \| \bA \|_{\text{B}, \infty, 1} = \max_{i \in [k]} \| \f(\bA_i) \|_1$ and $ \| \bA \|_{\text{B}, 1, 2} = \sum_{i \in [k]} \| \f(\bA_i) \|_2$. 

\section{Problem Description}
\label{sec:problem description}%

We define the skeleton $G_{\text{skel}}(V, E')$ of a directed graph $G(V, E)$ as an undirected graph which is constructed by removing directions from the edges in $E$, i.e., $(i,j) \in E'$ if and only if  $(i,j) \in E$ or  $(j, i) \in E$. Our goal is to recover $G_{\text{skel}}(V, E')$ from $N$ i.i.d. observations. We do not focus on recovering the orientation of the edges in $G$. However, readers should note that there exist techniques for obtaining a DAG given a skeleton. For example, Ordyniak et al.~\cite{ordyniak2013parameterized} showed that if the skeleton has bounded treewidth, then DAG recovery can be performed in polynomial time. Furthermore, given a skeleton of bounded treewidth and bounded maximum degree, DAG recovery is possible in linear time. 

\subsection{Our Main Assumption} Our approach is based on the following two intuitions. First, we assume that the parents and children of a node have a high influence on the original node. Second, it becomes easier to differentiate two parents (or two children) of a node when they are not highly correlated. Let $\nbr$ be a set containing indices of parents and children of node $r$. We consider the following characterization:
\begin{align*}
\bW_{\nbr}^* = \E_{\mathbb{X}}[\encode(X_{\bar r})_{\nbr} \encode(X_{\bar r})_{\nbr}^\T]^{-1} \E_{\mathbb{X}}[\encode(X_{\bar r})_{\nbr} \encode(X_r)^\T] 
\end{align*}
which tries to capture both of our intuitions mathematically. This quantity contains one block $\bW_i^* \in \real^{\rho_i \times \rho_r}$ for each parent and children of node $r$ and we require that $\bW_i^* \ne \bzero$. For binary variables, this requirement simply becomes 
\[ \E_{\mathbb{X}}[X_{\nbr} X_{\nbr}^\T]^{-1}\E_{\mathbb{X}}[X_{\nbr}X_r] \in (\real - \{0\})^{|\bS_r|} \ . \]
In later sections, we will formally build a mathematical foundation for our intuitions.  

\section{Modeling}
\label{sec:modeling}

In this section, we explain the construction of our mathematical model. We do not assume any specific conditional distribution for the nodes, thus modeling the problem becomes important for us. We also need to be careful while using categorical random variables in our mathematical model. We do not assume that categorical variables are ordinal, thus we do not want to introduce any artificial ordering while using them in our model.

\subsection{Substitute Model}
\label{subsec:substitute model}

Our approach is to recover the true parents and children of each node and then combine the results together to get the true skeleton of a Bayesian network. Before we start constructing a mathematical model, we need to understand certain aspects of our problem. We note that in other problems, such as compressed sensing~\cite{wainwright2009sharp}, the data generation process matches the estimation method. In contrast, in our setting, we assume that samples are generated according to a true Bayesian network, from unknown arbitrary conditional probability distributions for each node. This unavailability of a true model forces us to use a substitute model. Additionally, we encode the categorical random variables to use them in our model. After encoding, each variable is represented as a vector. In our discussions below, we will only discuss about recovering the parents and children of a single node. We keep in mind that we can combine our results for the nodes to recover the whole skeleton of the Bayesian network by simply taking a union bound over all the nodes. Considering the above, we can think of following general model for each node $r$,
\begin{align*}
\encode(X_r) \triangleq \calF(\encode(X_{\bar r}); \bW^*).
\end{align*}    
where $\calF$ is possibly a non-deterministic function and $\bW^*$ is a set of parameters. For our purpose, we choose the following form of $\calF$:
\begin{align*}
\encode(X_r) = {\bW^*}^\T \encode(X_{\bar r}) + \be
\end{align*}    
where $\bW^* \in \real^{\rho_{[n]\backslash r} \times \rho_r}$ is a parameter matrix. Note that $\be \in \real^{\rho_r}$ is not independent of $\encode(X_r), \encode(X_{\bar r})$ and $\bW^*$. We take $\bW^*$ to be a row partitioned block matrix by decomposing it into following blocks:
\begin{align*}
\bW^* = \begin{bmatrix}
\bW^*_1 \\ \vdots \\ \bW^*_{r - 1} \\ \bW^*_{r+1} \\ \vdots \\ \bW^*_n
\end{bmatrix}
\end{align*}  
where each $\bW^*_i \in \real^{\rho_i \times \rho_r} , \forall i \in [n], i \ne r$. We fix our choice of $\bW^*$ by defining the following optimization problem:
\begin{align}
\label{eq:opt problem}
\begin{matrix}
\bW^* &= &\arg\min_{\bW} & \frac{1}{2}\E_{\mathbb{X}}[\| \encode(X_r) - \bW^\T \encode(X_{\bar r}) \|_2^2] \\
&  & \text{such that} & \bW_i = \bzero, \forall i \in (\bpi_r \cup \bc_r)^c
\end{matrix}
\end{align}
or equivalently,
\begin{align*}
\begin{split}
\bW^* &= (\bW_{\nbr.}^*; \bzero) \\
 \bW_{\nbr.}^* &= \E_{\mathbb{X}}[\encode(X_{\bar r})_{\nbr} \encode(X_{\bar r})_{\nbr}^\T]^{-1} \E_{\mathbb{X}}[\encode(X_{\bar r})_{\nbr} \encode(X_r)^\T] 
\end{split}
\end{align*}
where the expectation is taken with respect to true data distribution $\prob(\mathbb{X})$ defined in equation~\eqref{eq:bayesian_network}. Optimization problem~\eqref{eq:opt problem} is introduced only for analysis purposes and it is not possible to be solved without knowing the true parents and children of node $r$. We note that each element of $\be$ is bounded. Let 

\begin{minipage}{.5\linewidth}
\begin{align}
\label{eq:sigma 1}
\| \be \|_{\infty} \leq 2\sigma  
\end{align}
\end{minipage}%
\begin{minipage}{.5\linewidth}
\begin{align}
\label{eq:mu 1}
\| \E_{\mathbb{X}}[|\be|] \|_{\infty} \leq \mu  \ .
\end{align} 
\end{minipage}

We emphasize that $\bW^*$ is not a true model parameter, i.e., we do not assume that the data follows a multivariate linear regression model. Instead, the substitute model allows us to find technical conditions with respect to the expectations of the products of encoded node pairs.

\subsection{Our Model}
\label{subsec:our model}

The substitute model, defined above for the infinite sample setting, acts as a benchmark model to perform qualitative analysis for our model in the finite sample setting. From equation~\eqref{eq:opt problem} which is defined for node $r$, it is clear that for node $i $ if $\bW_i^* \ne \bzero$, or equivalently, if $\| \f(\bW^*_i) \|_2 > 0$ then node $i$ is either a parent or a child of node $r$. This gives us the intuition to use $\ell_{1,2}$ regularization in order to encourage several blocks of the estimated matrix $\bW$ to be zero. In particular, our method would succeed if $\|\f(\bW_i)\|_2 = \bzero$ for all $i \in (\bpi_r \cup \bc_r)^c$. Let  $\hat\Loss(\bW)$ be the loss function defined as,
\begin{align}
\label{eq:lossN}
\begin{split}
\hat\Loss(\bW) \triangleq \frac{1}{2N} \| \encode(\bXr) - \encode(\bXnr) \bW \|_F^2 
\end{split}
\end{align}
Then we define the block $\ell_{1,2}$ regularized loss function as follows:   
\begin{align*} 
\hat f(\bW) &\triangleq \hat\Loss(\bW) + \lambdahat \| \bW \|_{\tB, 1, 2}  
\end{align*}
where $\lambdahat > 0$ is the regularization parameter. We recover weights $\bWhat$ for each node by minimizing $\hat f(\bW)$. The
optimization problem is defined as follows:
\begin{align} 
\label{eq:our_model} 
\bWhat = \min_{\bW} \hat f(\bW) \ .
\end{align}
We will show that under certain conditions we can use $\bWhat$ to determine the true parents and children of node $r$. Next, we define some terminology related to our models. We define the gradient and the Hessian for the loss function defined in equation~\eqref{eq:lossN} with respect to the parameters $\bW$. Note that equation~\eqref{eq:lossN} can be written as a function of $\f(\bW)$ and that the gradient and the Hessian can be easily computed with respect to $\f(\bW)$. For notational clarity, we will use matrix calculus to express the gradient while noting that this can easily be converted to the traditional form of the gradient by using a $\f(.)$ operation.
\begin{align*}
\begin{split}
\nabla \hat\Loss_{\bW}(\bW) &= \frac{1}{N} {\encode(\bXnr)}^\T \encode(\bXnr) \bW - \frac{1}{N} {\encode(\bXnr)}^\T \encode(\bXr) \\
\nabla^2 \hat\Loss_{\bW}(\bW) &= \begin{bmatrix}
\bHhat & \bzero & \dots & \bzero \\
\bzero & \bHhat & \dots & \bzero \\
\vdots & \vdots & \dots & \vdots \\
\bzero & \bzero & \dots & \bHhat 
\end{bmatrix} 
\end{split}
\end{align*} 
where ${\nabla^2 \hat\Loss_{\bW}(\bW) \in \real^{\rho_r \rho_{[n]\backslash r} \times \rho_r \rho_{[n]\backslash r} }}$  and  $\bHhat  = \frac{1}{N} {\encode(\bXnr)}^\T \encode(\bXnr) \in \real^{\rho_{[n]\backslash r} \times \rho_{[n]\backslash r}}$.

Analogously, we define a population version of $\bHhat$ as, 
\begin{align*}
\begin{split}
\bH = \E_{\mathbb{X}} \big[\encode(X_{\bar r})\encode(X_{\bar r})^\T\big] \ .
\end{split}
\end{align*}
Our choice of loss function ensures that $\bH$ does not depend on $\bW^*$. Thus, any assumptions on $\bH$ only 
correspond to restrictions on the data distribution $\prob(\mathbb{X})$ defined in equation~\eqref{eq:bayesian_network}. 

Usually if we have the knowledge of the true data generation process or the conditional probability distribution of the nodes, then we can learn parameters of the distribution by minimizing a well defined empirical loss. In our case, we do not have this information. We circumvent this issue by defining a substitute model for our problem. For each node, we assign a matrix of non-zero \emph{surrogate} parameters for its neighbors. For all the other nodes which are not neighbors, this parameter matrix is zero. Then we construct a substitute quadratic loss function with respect to the surrogate parameters. This choice of loss function is crucial as unlike other loss functions (such as the logistic loss) the Hessian of the quadratic loss becomes independent of the surrogate parameters. This ensures that any technical condition on the Hessian translates directly to a condition on the expectations of the products of encoded node pairs.

\section{Main Result}
\label{sec:main result}
In this section, we state our main theoretical result. Recall that the general problem of structure learning of the Bayesian network is NP-hard~\cite{chickering2004large}. Thus rather than learning a general class of Bayesian networks, we focus on the networks which satisfy certain technical assumptions. 

\subsection{Technical Assumptions}
\label{subsec:assumption}

In this subsection, we establish the sufficient technical conditions for the perfect recovery of the parents and children for each node. Our first goal is to always recover a unique set of parents and children. In order to achieve this task, we require that our optimization problem defined in equation \eqref{eq:our_model} has a unique solution. Our first assumption on the data distribution ensures a unique solution for the optimization problem \eqref{eq:our_model}. Recall that each block in the row partitioned parameter matrix $\bW$ and $\bW^*$ corresponds to one node. Each of these block $i$ contains $\rho_i$ row indices. We collect these row indices corresponding to the parents and children of node $r$ in a set $\nbr$. Formally,
\[
\bS_r = \bigcup_{i \in \bpi_r \cup \bc_r} \{ \rho_{[i-1]} + 1, \dots, \rho_{[i]} \}
\]
We define $\nbr^c$ as the row indices corresponding to all nodes except node $r$ as well as its parents and children:
\[
\bS_r^c = [\rho_{[n]}] - \bS_r - \{ \rho_{[r-1]} + 1, \dots, \rho_{[r]} \}
\]
Using the above definitions, we state our first assumption. 
\begin{assumption}[Positive Definiteness of Hessian] 
	\label{assum:positive_definiteness_assumption}    
	For each node $r$, $ \bH_{\bS_r \bS_r} \succ 0 $ or equivalently, $ \eval_{\min}(\bH_{\bS_r \bS_r }) \geq C > 0 $. 
\end{assumption}
where $C$ is some positive constant and $\eval_{\min}(.)$ denotes the smallest eigenvalue. We solve the optimization problem using a finite number of samples. Thus, we would like our assumptions to hold in the finite sample setting. The next lemma shows that if we have $ N > O(\rho_{\bpi_r \cup \bc_r}^2 \log \rho_{\bpi_r \cup \bc_r} ) $ samples and Assumption \ref{assum:positive_definiteness_assumption} is satisfied, then $\bHhat_{\bS_r \bS_r} \succ 0$ with high probability.
\begin{lemma} 
	\label{lemma:sample_positive_definiteness} 
	If $\bH_{\bS_r \bS_r} \succ 0$ then $\bHhat_{\bS_r \bS_r} \succ
	0$ with probability at least 
	$ 1 - 2 \exp(- \frac{\delta^2 N }{8 \rho_{\bpi_r \cup \bc_r}^2} + 2 \log \rho_{\bpi_r \cup \bc_r}) $.
\end{lemma}
(See Appendix \ref{proof:sample_positive_definiteness} for detailed proof.)

As the second requirement, we want to limit the influence of the nodes which are neither the parents nor the children of node $r$ on the parents and children of node $r$. This is represented as a ``mutual incoherence'' condition. We will define $ \bQ =  \bH_{\bS_r^c\bS_r}\bH_{\bS_r \bS_r}^{-1}$ as a row partitioned block matrix consisted of blocks $\bQ_i \in \real^{\rho_i \times \rho_{\bpi_r \cup \bc_r}}\ \forall i\ \in (\bpi_r \cup \bc_r)^c$. We can formally state our second assumption using a block matrix norm on $\bQ$ as follows.
\begin{assumption}[Mutual Incoherence]
	\label{assum:mutual_incoherence}
	For each node $r$, $ \| \bQ \|_{\tB, \infty, 1} = \|\bH_{\bS_r^c\bS_r}\bH_{\bS_r \bS_r}^{-1}\|_{\tB, \infty, 1} \leq 1 - \alpha $ for some $\alpha \in (0,1]$. 
\end{assumption} 
As with the Assumption \ref{assum:positive_definiteness_assumption}, we would again like Assumption \ref{assum:mutual_incoherence} to hold in the finite sample setting. In the next lemma we show that if we have sufficient number of samples, then the mutual incoherence  in the population regime ensures that mutual incoherence also holds in the finite-sample regime.  
\begin{lemma}
	\label{lemma:mutual_incoherence}
	If $|||\bH_{\bS_r^c\bS_r}\bH_{\bS_r \bS_r}^{-1}|||_{\tB, \infty, 1} \leq 1 - \alpha$ for $\alpha \in (0, 1]$ then $|||\bHhat_{\bS_r^c \bS_r}\bHhat_{\bS_r \bS_r}^{-1}|||_{\tB, \infty, 1} \leq 1 - \alpha$ with probability at least $ 1 - O( \exp(\frac{-K N}{ \bar{\rho}^2  \rho_{\bpi_r \cup \bc_r}^3} + \log \rho_{(\bpi_r \cup \bc_r)^c} + \log \rho_{\bpi_r \cup \bc_r} )) $ for some $K > 0$. 
\end{lemma}
(See Appendix \ref{proof:mutual_incoherence} for detailed proof.)

\paragraph{Discussion on the technical assumptions.}
These assumptions have been used in the literature before. 
\begin{itemize}
	\item Mutual incoherence has been used for other estimation problems such as compressed sensing~\cite{wainwright2009sharp}, Markov random fields~\cite{ravikumar2010high}, non-parametric regression~\cite{ravikumar2007spam}, diffusion networks~\cite{daneshmand2014estimating}, among others. 
	\item Assumption \ref{assum:positive_definiteness_assumption} is readily satisfied for commonly used encoding schemes such as dummy encoding and unweighted effects encoding under very weak conditions (See Appendix \ref{sec:pos def normal encoding}). 
	\item Regarding Assumption \ref{assum:mutual_incoherence}, we found experimentally that mutual incoherence is more frequently satisfied with respect to the parents and children, than with respect to the Markov blanket (See Appendix \ref{sec:on markov blanket}).
	\item Theorem \ref{theorem:main_result} in the next section provides a technical condition for setting $\lambdahat$ (without the need of cross-validation) and assumes a minimum magnitude of $\|\f(\bW_i^*)\|_2, \forall i \in \bpi_r \cup \bc_r$ for exact recovery of parents and children. Analogous technical assumptions have been made for other problems\cite{wainwright2009sharp, ravikumar2010high, ravikumar2011high,  daneshmand2014estimating}.
	\item Finally, these assumptions are only in place to provide formal guarantees and our algorithm can be run even for datasets which do not satisfy any of these assumptions (See experimental results in Section~\ref{sec:experimental results}). 
\end{itemize}

\subsection{Statement of Main Theorem}
\label{subsec:main theorem}
Using Assumptions \ref{assum:positive_definiteness_assumption} and \ref{assum:mutual_incoherence}, we state our main result below. 

\begin{theorem} 
	\label{theorem:main_result}
	Consider a Bayesian network $G(V, E)$ with categorical random variables such that for each node $r$,  Assumptions \ref{assum:positive_definiteness_assumption} and \ref{assum:mutual_incoherence} are satisfied. Suppose that for each node $r$ the regularization parameter $\lambdahat$ satisfies the following condition: 
	\begin{align} 
	\label{eq:lambda_condition}
	\begin{split} 
	&\lambdahat >\frac{4}{\alpha} \sqrt{ \frac{\rho_r}{N}} \max \big( (1 - \alpha) ( \sqrt{2\sigma^2 \log (\rho_{\bpi_r \cup \bc_r} \rho_r)}+ \mu) \\ 
	& ,  \sqrt{\bar{\rho}}( \sqrt{2\sigma^2 \log ( |(\bpi_r \cup \bc_r)^c| \bar{\rho} \rho_r)} +  \mu) \big)
	\end{split}
	\end{align} 
	where $\sigma$ and $\mu$ are defined according to equations~\eqref{eq:sigma 1} and \eqref{eq:mu 1}. Further, assume that $N > \bar{\rho}^2  \rho_{\bpi_r \cup \bc_r}^3 \log \rho_{[n]}$ then the following properties hold true with probability at least $1 - \exp(- K N \lambdahat^2 )$ for some $K > 0$ independent of $N, n$ and $|\bpi_r \cup \bc_r|$ simultaneously for all $r \in [n]$.
	\begin{enumerate}
		\item For every $r\in[n]$, the block $\ell_{1,2}-$regularized optimization problem~\eqref{eq:our_model} has a unique
		solution.
		\item For every $r \in [n]$, the solution to the optimization problem~\eqref{eq:our_model} excludes all the edges which are neither parent nor child of the node $r$, i.e., $\| \f(\bWhat_i)\|_2 = 0, \forall i \in
		(\bpi_r \cup \bc_r)^c$.
		\item If $\min_{i \in \bpi_r \cup \bc_r} \|\f(\bW^*_i)\|_2 > \frac{4\bar{m}}{C} (   \frac{\alpha}{4 (1 - \alpha)} + \sqrt{\rho_r} +  1 ) \sqrt{|\nbr|} \lambdahat$ for the setup defined in the substitute optimization problem \eqref{eq:opt problem}, then we recover the true parents and children for each node.
		\item Subsequently, the recovered skeleton $ \hat G_{\text{skel}}(V, \hat E') = G_{\text{skel}}(V,E')$.  
	\end{enumerate}
\end{theorem}

We prove Theorem \ref{theorem:main_result} in Appendix \ref{subsec:proof of theorem 1} by using a primal-dual witness construction. This approach has been previously used by \cite{wainwright2009sharp, ravikumar2010high, ravikumar2011high,  daneshmand2014estimating}. The primal-dual witness method requires a priori knowledge of the true parents and children for node $r$ and thus it is not a practical way to solve the optimization problem~\eqref{eq:our_model}. We only use it as a theoretical proof technique to establish statistical bounds for our result.

\subsection{Illustrative Example}
\label{subsec:illustrative example}

In order to illustrate our assumptions, consider the binary Bayesian network of four nodes in Figure \ref{fig:small bn} where each node  $X_i \in \{ \text{False}, \text{True} \}$, $ \forall i\in \{1,2,3,4\}$.  
\begin{figure}[h]
	\centering
	\begin{tikzpicture}[scale=0.7]
	\node[] at (0, 0)   (x1) {\small $1$};
	\node[] at (-1, -1)   (x2) {\small $2$};
	\node[] at (1, -1)   (x3) {\small $3$};
	\node[] at (0, -2)   (x4) {\small $4$};
	\draw[very thick, ->] (x1) -- (x2);
	\draw[very thick, ->] (x2) -- (x4);
	\draw[very thick, ->] (x3) -- (x4);
	\draw[very thick] (0, 0) circle (0.3cm); 
	\draw[very thick] (-1, -1) circle (0.3cm); 
	\draw[very thick] (1, -1) circle (0.3cm); 
	\draw[very thick] (0, -2) circle (0.3cm); 
	\end{tikzpicture}	
	\caption{\label{fig:small bn}Binary Bayesian Network}
\end{figure}
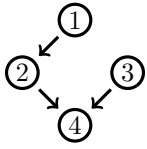

We use unweighted effects encoding, i.e., $\encode(\text{False}) = -1, \encode(\text{True}) = 1, \forall i \in \{1,2,3,4\}$. Note that since each node takes 2 values, Assumption \ref{assum:mutual_incoherence} reduces to  $|||\bH_{\nbr^c\nbr} \bH_{\nbr\nbr}^{-1}|||_{\infty,\infty} < 1 - \alpha$ for some  $\alpha \in (0, 1]$. We assume that $|\E[\encode(X_i)\encode(X_j)]| < 1, \forall i, j \in \{1,2,3,4\}, i \ne j$. Furthermore, we assume that $\E[\encode(X_1) \encode(X_4)] = q$, $ \E[\encode(X_1)\encode(X_2)] = \E[\encode(X_2) \encode(X_4)] =  \E[\encode(X_3) \encode(X_4)] = p$ and $\E[\encode(X_1)] = \E[\encode(X_3)] = 0$. Then the Assumption \ref{assum:mutual_incoherence} is equivalent to the condition that $|p| + |q| < 1$ (See Appendix \ref{subsec:discussion on assumption}).  

The third statement of Theorem \ref{theorem:main_result} requires that for every node $i \in \bpi_r \cup \bc_r$, $\|\f(\bW^*_i)\|_2$ should be sufficiently away from zero. By computing  $\min_{i \in \bpi_r \cup \bc_r} \|\f(\bW^*_i)\|_2 $ for each node we can conclude that the third statement holds as long as $ \min (|p|, |\frac{p}{1 + q}|)$ is not too close to zero(See Appendix \ref{subsec:on wstarmin}).

\subsection{Sample And Time Complexity}
\label{sec:sample and time complexity} 

If we have $N > \bar{\rho}^2  \rho_{\bpi_r \cup \bc_r}^3 \log \rho_{[n]}$ and Assumption \ref{assum:positive_definiteness_assumption} and \ref{assum:mutual_incoherence} are satisfied for every node then all our high probability statements are valid for every node $r$. Taking a union bound over $n$ nodes only adds a factor of $\log n$. Thus the sample complexity for our method is $O( \bar{\rho}^2    \rho_{\bpi_r \cup \bc_r}^3 \log  \rho_{[n]})$. As for the time complexity, we can formulate the block $\ell_{12}$-regularized multi-variate regression problem as a second order cone programing problem~\cite{obozinski2011support} which can be solved in polynomial time by interior point methods~\cite{boyd2004convex}. 

\section{Experimental Results}
\label{sec:experimental results}

We performed three sets of experiments to validate our theoretical results. First, we conducted experiments on synthetic data. Second, we compared our method with other well known methods on benchmark Bayesian networks. Finally, we tested our method on real world datasets. We measure quality of recovery by computing precision and recall. Higher precision implies that we only recover true edges while higher recall implies that all the true edges are recovered. They are formally defined as below,
\begin{align*}
\text{Precision} &=\frac{\sum_{r=1}^n |(\hat{\bpi_r} \cup \hat{\bc_r}) \cap (\bpi_r \cup \bc_r)|}{ \sum_{r=1}^n |(\hat{\bpi_r} \cup \hat{\bc_r})|} \\
 \text{Recall} &= \frac{\sum_{r=1}^n |(\hat{\bpi_r} \cup \hat{\bc_r}) \cap (\bpi_r \cup \bc_r)|}{  \sum_{r=1}^n |(\bpi_r \cup \bc_r)|}
\end{align*}
where $(\hat{\bpi_r} \cup \hat{\bc_r})$ is the recovered support set (both parents and children). We compare performance of various methods by computing $F_1$-score as following:
\begin{align*}
	F_1 \text{-score} = \frac{2 \times \text{Precision} \times \text{Recall} }{\text{Precision} + \text{Recall}}
\end{align*}

\subsection{Experiments With Synthetic Data}
\label{subsec:synthetic experiments}

We verify our theoretical results by running our method on synthetic data.  We conduct experiments on Bayesian networks with $n = 20, 200$ and $500$ nodes. Each node of Bayesian network can take $k = 4$ categorical values. For each $n$, we generate $N = 10^{CP} \log (k - 1) n $ i.i.d. samples. $CP$ is a control parameter and is varied to generate different number of samples. Using these samples, our method learns the skeleton of the Bayesian network by performing block $\ell_{12}$-regularized multivariate regression for each node. The regularization parameter $\lambdahat$ for each regression problem is set proportional to $\sqrt{\frac{\log (k - 1) n}{N}}$ until it becomes smaller than a constant $\delta$. This matches our condition in Theorem \ref{theorem:main_result} where initially $\sqrt{\frac{\log (k-1) n}{N}}$ dominates but then a constant term dominates as $N$ gets large. The quality of skeleton recovery is measured by computing precision and recall which we report in Figure \ref{fig:test}. Each data point in Figure \ref{fig:test} denotes averaged value across $5$ independent experiments.  

\paragraph{Arbitrary conditional probability tables.} First, we pick a causal order for nodes uniformly at random. Then we construct a DAG by allowing each node to have an edge with any preceding node in the ordering with $0.5$ probability. We induce sparsity by performing transitive reduction~\cite{aho1972transitive} of the DAG.
Each node is then assigned a conditional probability table (CPT) conditioned on its parents. The entries in CPTs are chosen uniformly at random from $[0.1, 0.9]$. We use CPT of the node to sample its value given its parents.  

\begin{figure}[ht]
\centering
\begin{subfigure}{.5\textwidth}
\centering
\includegraphics[width=\linewidth]{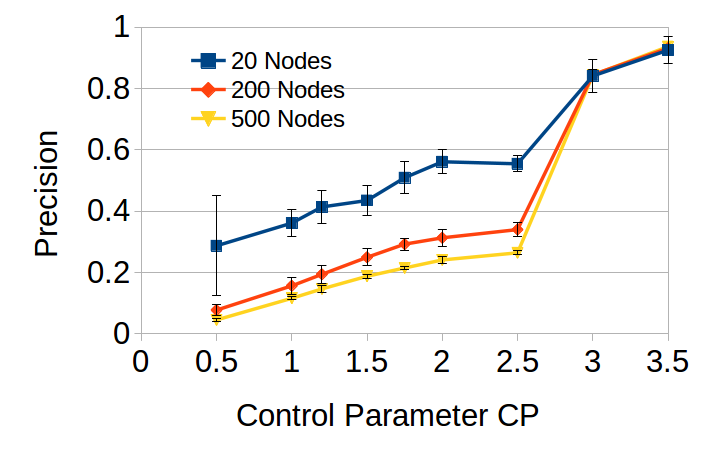}
\caption{Precision vs. CP ($k=4$)}
\label{fig:sub1}
\end{subfigure}%
\begin{subfigure}{.5\textwidth}
\centering
\includegraphics[width=\linewidth]{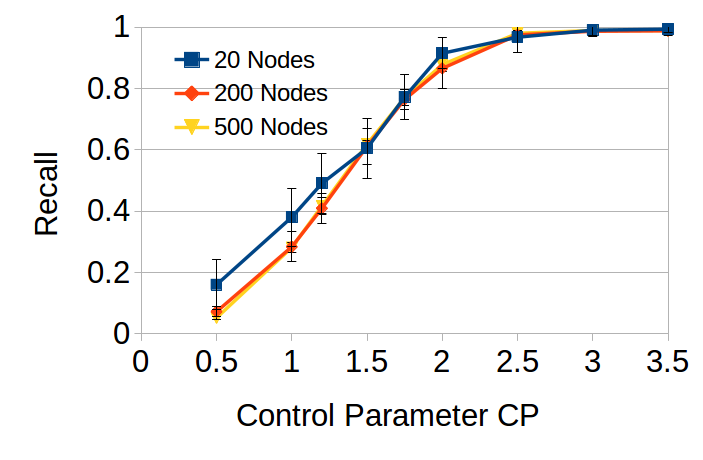}
\caption{Recall vs. CP ($k=4$)}
\label{fig:sub2}
\end{subfigure}
\caption{Plots of Precision And Recall Versus The Control Parameter $CP$ For Bayesian Networks On $n = 20, 200$ And $500$ Nodes With $N = 10^{CP} \log (k-1)n$ Samples}
\label{fig:test}
\end{figure}

Figure \ref{fig:sub1} and \ref{fig:sub2} show the precision and recall respectively for our method with increasing number of samples. In Figure \ref{fig:sub1}, we see that precision approaches one, with enough samples. In Figure \ref{fig:sub2}, recall also approaches one as we increase number of samples. Notice that the different curves for different number of nodes ($n = 20, 200$ and $500$) line up with one another quite well. This matches with our theoretical results and shows that for a Bayesian network with a constant degree, our method can efficiently recover the skeleton with $N > O(\log (k-1) n)$.  

\begin{table*}[!ht]
	\caption{ $F_1$-Scores and standard errors at $95\%$ confidence level on benchmark Bayesian networks. First two columns compare skeletons. Last six columns compare DAGs.}
	\label{table:benchmark compare f1}
	\begin{center}
	\setlength{\tabcolsep}{2pt}
	\small
	\begin{tabular}{lc|c||c|c|c|c|c|c}
		\toprule
		& \multicolumn{8}{c}{{\bf $\bm{F_1}$-Score}}                   \\
		\cmidrule(r){2-9}
		{\bf Network}  & Our & MMPC & Our Method & MMHC& Greedy  & Sparse & Optimal & Exact \\
		($\bm{n = |V|}$) & Method &  & + Greedy & & Search & Candidate & Reinsertion & LP \\
		\midrule
		alarm (37) & $0.74 \pm 0.01$ & $0.90 \pm 0.01$ &\cellcolor{blue!20}$0.92 \pm 0.02$ &\cellcolor{blue!20}$0.93 \pm 0.02$ & $0.87 \pm 0.07$ & $0.54 \pm 0.06$ &  $0.86 \pm 0.04$ & $0.82 \pm 0.02$ \\
		andes (223) &\cellcolor{blue!20}$0.77 \pm 0.01$ &\cellcolor{blue!20} $0.79 \pm 0.01$ & $0.82 \pm 0.01$ &$0.88 \pm 0.01$ & $0.79 \pm 0.02$ & $0.05 \pm 0.02$ &  $0.74 \pm 0.01$ & $0.67 \pm 0.01$ \\
		barley (48) & $0.37 \pm 0.01$ & $0.60 \pm 0.01$ &\cellcolor{blue!20} $0.74 \pm 0.06$ &\cellcolor{blue!20}$0.73 \pm 0.05$ & $0.61 \pm 0.05$ & $0.61 \pm 0.04$ &  $0.67 \pm 0.1$ & $0.77 \pm 0.01$ \\
		carpo (60) & $0.74 \pm 0.01$ & $0.78 \pm 0.01$ &\cellcolor{blue!20} $0.89 \pm 0.01$ &\cellcolor{blue!20}$0.86 \pm 0.03$ & $0.78 \pm 0.02$ & $0.15 \pm 0.04$ &  $0.82 \pm 0.02$ & $0.69 \pm 0.02$ \\
		child (20) &\cellcolor{blue!20} $0.94 \pm 0.01$ &\cellcolor{blue!20} $0.93 \pm 0.01$ & $0.98 \pm 0.00$ &$1 \pm 0.00$ &$1 \pm 0.00$ & $0.79 \pm 0.11$ &  $0.96 \pm 0.05$ & $1 \pm 0.00$ \\
		hailfinder (56) &\cellcolor{blue!20} $0.57 \pm 0.02$ &\cellcolor{blue!20} $0.58 \pm 0.01$ & $0.65 \pm 0.02$ &$0.75 \pm 0.14$ & $0.68 \pm 0.15$ & $0.46 \pm 0.08$ &  $0.71 \pm 0.1$ & $0.68 \pm 0.04$ \\
		mildew (35) & $0.30 \pm 0.01$ & $0.45 \pm 0.02$ &\cellcolor{blue!20} $0.83 \pm 0.01$ &\cellcolor{blue!20}$0.77 \pm 0.00$ & $0.70 \pm 0.08$ & $0.68 \pm 0.04$ &  $0.65 \pm 0.08$ & $0.72 \pm 0.00$ \\
		water (32) & $0.59 \pm 0.02$ & $0.63 \pm 0.01$ &\cellcolor{blue!20}$0.61 \pm 0.02$ &\cellcolor{blue!20}$0.60 \pm 0.05$ & $0.51 \pm 0.06$ & $0.28 \pm 0.03$ &  $0.62 \pm 0.06$ & $0.61 \pm 0.03$ \\
		win95pts (76) & $0.67 \pm 0.01$ & $0.77 \pm 0.03$ &\cellcolor{blue!20} $0.76 \pm 0.01$ &\cellcolor{blue!20}$0.79 \pm 0.03$ & $0.62 \pm 0.02$ & $0.13 \pm 0.06$ &  $0.62 \pm 0.04$ & $0.59 \pm 0.02$ \\
		\bottomrule
	\end{tabular}
	\end{center}
\end{table*} 

\subsection{Experiments On Benchmark Networks}
\label{subsec:experiments on benchmark networks}

We compared the performance of our method with state-of-the-art techniques by running experiments on benchmark Bayesian networks, which are publicly available at \url{http://compbio.cs.huji.ac.il/Repository/networks.html} and \url{http://www.bnlearn.com/bnrepository/}. The experiments were conducted by generating $5$ independent instances of $5000$ samples using the original conditional probability tables of the benchmark networks. The regularization parameter for node $i$, which can take $k_i$ categorical values, is chosen to scale with $\lambdahat =  c_1 \sqrt{\frac{\log (k_i-1)(\frac{\sum_{j=1}^{n} k_j}{n} - 1) n}{N}} + c_2$ for constants $c_1$ and $c_2$. We report the average $F_1$-score across $5$ independent runs. We compared our method with the max-min parent and children (MMPC) algorithm~\cite{tsamardinos2003time} which also returns an undirected skeleton.  To compare our method with techniques that produce DAGs (such as max-min hill climbing (MMHC) \cite{tsamardinos2006max}, greedy search,
integer linear programming (LP) \cite{cussens2012bayesiancut} , sparse candidate \cite{friedman1999learning}, and
optimal reinsertion operator \cite{moore2003optimal}) and to provide further insight, we oriented the edges in the skeleton produced by our method by using greedy hill-climbing search with tabu list. This setup is similar to the one used in \cite{tsamardinos2006max}. 

In Table~\ref{table:benchmark compare f1}, we observe that the performance of our method is comparable to MMPC. Our method also performs comparably to other state-of-the-art techniques when we run our method in conjunction with greedy hill climbing.  

\begin{table}[!h]
	\caption{Negative log-likelihood and standard errors at $95\%$ confidence level on real world datasets.}
	\label{table:realworld compare}
	\begin{center}
		\begin{tabular}{ll|l|l}
			\toprule
			& \multicolumn{3}{c}{ \bf Negative Log-likelihood}                   \\
			\cmidrule(r){2-4}
			{\bf Network}  & Our Method & MMHC & Greedy  \\
			($n = |V|$)  & + Greedy &  &   \\
			\midrule
			dna	(180) & $80.37 \pm 0.24$ & $80.57\pm0.24$ & $80.45 \pm 0.25$ \\
			moviereview (1001)& $339.27\pm8.85$ &$340.60\pm9.03$ & $343.64\pm9.28$\\
			retail (135) & $10.88\pm0.21$ & $10.88\pm0.21$ & $10.88\pm0.21$ \\
			audio(125) &$40.50 \pm	0.57$ &	$40.28\pm	0.57$&	$40.27\pm	0.57$ \\
			autos (26) &$18.61\pm2.98$ &$24.45\pm1.74$ &$18.62\pm3.27$\\
			jester(100) & $53.78\pm 0.39$	& $53.53 \pm0.40$ & $53.53 \pm 0.40$ \\
			netflix	(100) & $57.02\pm0.23$ & $56.85\pm0.23$ & $56.85\pm0.23$ \\
			r52	(889) & $90.46\pm3.43$ & $86.92 \pm 3.46$ & $87.30\pm3.49$ \\
			student-por (33) &$32.22\pm0.72$ & $33.88\pm0.69$&$32.08\pm0.73$\\
			tmovie(500)  & $55.33\pm4.63$ & $54.59\pm4.65$ & $54.96\pm4.75$ \\
			webkb (839)  & $159.39 \pm 6.34$ & $156.89\pm6.28$ & $157.75\pm6.46$ \\
			promoters (58) &$78.03 \pm 1.84$ & $79.01\pm2.10$ & $79.01\pm2.10$ \\
			sponge(45) & $26.36\pm4.64$ &$29.50\pm3.54$ & $25.49\pm4.65$\\
			triazines(59) & $13.78\pm1.75$&$13.89\pm1.60$ &$12.83\pm2.10$ \\
			wiki4he(53) & $57.99\pm1.29$ & $58.08\pm1.34$&$58.05\pm1.34$\\
			\bottomrule
		\end{tabular}
	\end{center}
\end{table}

\subsection{Experiments With Real World Datasets}

Finally, we conducted experiments on the real world datasets. For our experiments, we picked a mixture of binary and discrete real world datasets from \cite{malone2015impact} and \cite{van2012markov}.We divided the datasets in \cite{malone2015impact} into training and testing sets. The datasets provided in \cite{van2012markov} have already been divided into training and testing sets by the original authors. The training data was fed into the different algorithms. For real world datasets, we do not have access to any underlying true Bayesian network structure, thus performance is measured by the negative log-likelihood of samples in the testing set. Since our method only recovers the skeleton, in order to measure likelihood, edges were oriented using a greedy hill-climbing search with tabu list.In Table \ref{table:realworld compare}, we observe that performance of our method is similar to MMHC and greedy search. 

\section{Concluding Remarks}
\label{sec:conclusions}

We propose a method for exact structure recovery of discrete Bayesian network under some technical conditions. It runs in polynomial time and has polynomial sample complexity. We neither assume any specific data generation process nor do we impose any direct assumptions on the conditional probability distribution of the nodes. Rather, we control the interaction between node pairs with our assumptions. In practice, our method can be used for any discrete Bayesian network irrespective of whether the assumptions are satisfied, albeit without any guarantees.

\vfill
\onecolumn
\appendix 

\section{On Assumption \ref{assum:positive_definiteness_assumption} }
\label{sec:pos def normal encoding}

We show that Assumption \ref{assum:positive_definiteness_assumption} holds with very mild conditions.
\begin{lemma}
	\label{lem:pos def normal encoding}
	Consider a matrix $\bM$ with each row $i$ being one realization $\encode(x_{\bar r}^i)_{\nbr}$ of $\encode(X_{\bar r})_{\nbr}$. If each $\encode(x_{\bar r}^i)_{\nbr}$ occurs with a probability $p_i > 0$ then Assumption \ref{assum:positive_definiteness_assumption} holds as long as columns of matrix $\bM$ are linearly independent.   
\end{lemma}
\begin{proof}
	We show that Assumption \ref{assum:positive_definiteness_assumption} always holds for commonly used encoding scheme such as dummy encoding and effects encoding as long as $\encode(X_{\bar r})_{\nbr}$ takes all its realizations with some positive probability. To show this we consider the following loss function, 
	\begin{align*}
	\Loss(\bW_{\nbr .}) \triangleq \frac{1}{2} \E[\sum_{i=1}^{m_r -1}( \encode(X_{\bar r})_{\nbr}^\T \bW_{\nbr i} - \encode(X_r)_i)^2]
	\end{align*} 
	We will prove that $\nabla_{\bW_{\nbr .}}^2 \Loss(\bW_{\nbr .}) \succ 0$ . To do this, we restrict $\Loss(\bW_{\nbr .})$ to a line by taking $\bW_{\nbr .} = \bW^0_{\nbr .} + t \bW_{\nbr .}^1$ for any $t \in \{-\infty, \infty \}$ such that $\bW_{\nbr .}^1 \ne \bzero$. Then,
	\begin{align*}
	g(t) &\triangleq \frac{1}{2} \E[ \sum_{i=1}^{m_r -1}( \encode(X_{\bar r})_{\nbr}^\T (\bW_{\nbr i}^0 + t \bW_{ \nbr i}^1) - \encode(X_r)_i)^2 ]\\
	\frac{d^2 g(t)}{dt^2} &= \E[\sum_{i=1}^{m_r -1} (\encode(X_{\bar r})_{\nbr}^\T \bW_{\nbr i}^1)^2] 
	\end{align*}
	We assume that each realization $\encode(x_{\bar r}^j)_{\nbr}$ of $\encode(X_{\bar r})_{\nbr}$ happens with a probability $p_j > 0, \forall j \in [\prod_{i \in \bpi_r \cup \bc_r} m_i]$. Then,
	\begin{align*}
	\frac{d^2 g(t)}{dt^2} &= \sum_{j=1}^{\prod_{i \in \bpi_r \cup \bc_r} m_i} p_j \sum_{i=1}^{m_r -1} (\encode(x_{\bar r}^j)_{\nbr}^\T \bW_{\nbr i}^1)^2
	\end{align*}  
	Since $p_j > 0$, it follows that $\frac{d^2 g(t)}{dt^2} = 0 \iff \encode(x_{\bar r}^j)_{\nbr}^\T \bW_{\nbr i}^1 = 0, \forall i \in [\rho_r],j \in [\prod_{i \in \bpi_r \cup \bc_r} m_i]$. This can be equivalently written as,
	\begin{align*}
	\begin{bmatrix}
	\encode(x_{\bar r}^1)^\T \\
	\encode(x_{\bar r}^2)^\T \\
	\vdots \\
	\encode(x_{\bar r}^{\prod_{i \in \bpi_r \cup \bc_r} m_i})^\T 
	\end{bmatrix} \bW_{\nbr i}^1 \triangleq \bM \bW_{\nbr i}^1 = \bzero, \quad \forall i \in  [m_r - 1]
	\end{align*} 
	If we take encoding scheme to be dummy encoding or unweighted effects encoding then the above holds if and only if $\bW_{\nbr i}^1 = \bzero,  \forall i \in  [m_r - 1]$. This implies $\bW_{\nbr}^1 = \bzero$ which is not possible. Since choice of $t$ is completely arbitrary, it follows that $\nabla_{\bW_{\nbr .}}^2 \Loss(\bW_{\nbr .}) \succ 0$. Now, 
	\begin{align*}
	\nabla_{\bW_{\nbr .}}^2 \Loss(\bW_{\nbr .}) &= \begin{bmatrix}
	\bH_{\nbr\nbr} & \bzero & \dots & \bzero \\
	\bzero & \bH_{\nbr\nbr} & \dots & \bzero \\
	\vdots & \vdots & \dots & \vdots \\
	\bzero & \bzero & \dots & \bH_{\nbr\nbr} 
	\end{bmatrix}
	\end{align*}
	which is positive definite if and only if $\bH_{\nbr\nbr} \succ 0$.
\end{proof}

\section{On Mutual Incoherence With Markov Blanket}
\label{sec:on markov blanket}

We compared mutual incoherence assumption defined on the parents and children of a node (MIPC) in Assumption \ref{assum:mutual_incoherence} with the one defined on its Markov blanket (MIMB). We created $10$ different synthetic Bayesian networks on $n = 100, 500$ and $1000$ binary nodes. The nodes were assigned conditional probability tables (CPTs) with entries between $[0.1, 0.9]$. We generated $5000$ samples for each experiment and then computed validity of Assumption \ref{assum:mutual_incoherence} for each network. The results of the experiment are listed in the Table \ref{table:assum2 check}.

\begin{table}[h]
	\vspace{-0.1in}   \caption{Comparison of mutual incoherence assumption on different supports} 
	\label{table:assum2 check}
	\begin{center}
		\begin{tabular}{lllll}
			\toprule
			Number & Max  & Support : Parent and Children  & Support : Markov Blanket &  MIPC is   \\
			of nodes  &  Degree & (MIPC) holds & (MIMB)  holds &  weaker than MIMB\\
			\midrule
			100  & 7 & 99.7\% & 97.6\% & 91.9\% \\
			500 & 7 & 99.76\% & 96.82\% & 93.26\% \\
			1000  & 7 & 99.62\% & 95.95\% & 92.97\% \\
			\bottomrule
		\end{tabular}
	\end{center}
\end{table}  

We see that mutual incoherence assumption defined on parents and children holds more often than mutual incoherence assumption defined on Markov blanket. Also, it is easier to fulfill mutual incoherence assumption with parents and children as support than the case when Markov blanket is used as support. This motivates us to define mutual incoherence assumption on parents and children as support.

\section{Proof of Lemma \ref{lemma:sample_positive_definiteness}} 
\label{proof:sample_positive_definiteness}
\begin{proof} 
	$\bH_{\bS_r \bS_r} \succ 0$ can be equivalently written as $\eval_{\min}(\bH_{\bS_r \bS_r}) \geq C > 0$
	where $\eval_{\min}$ denotes the minimum eigenvalue. We define $Z_{jk}$ as,
	
	\begin{align*}
	\begin{split}
	[\bHhat - \bH]_{jk} &= Z_{jk} \\
	&= \frac{1}{N} \sum_{l=1}^N (\encode(\bX_{\bar r})_j^l \encode(\bX_{\bar r})_k^l - \E_{\mathbb{X}}[\encode(X_{\bar r})_j \encode(X_{\bar r})_k] ) \\
	&= \frac{1}{N}\sum_{l=1}^N Z_{jk}^l
	\end{split}
	\end{align*}
	 Note that $Z_{jk}^l$ are i.i.d. random variables across $l=[N]$ with zero mean.  Furthermore, $ | Z_{jk}^l | \leq 2 $ as we assumed $\calB \in \{-1, 0, 1\}$. Thus $Z_{jk}$ can be treated as a subGaussian random variable and by using the Azuma-Hoeffding~\cite{hoeffding1963probability} inequality we can write,
	\begin{align} 
	\label{eq:azuma} 
	\begin{split} 
	\prob[(Z_{jk})^2 \geq \epsilon^2] &= \prob\big[ | \frac{1}{N} \sum_{l=1}^N Z_{jk}^l | \geq \epsilon
	\big], \\ 
	&\leq 2 \exp(\frac{-\epsilon^2 N}{8}) \ . 
	\end{split} 
	\end{align}
	\noindent Now,
	\begin{align*} 
	\begin{split} 
	\eval_{\min}(\bH_{\bS_r \bS_r}) &= \min_{\|\bx\|_2 = 1} \bx^\T \bH_{\bS_r \bS_r} \bx \\ 
	&= \min_{\|\bx\|_2 = 1} \big(\bx^\T \bHhat_{\bS_r \bS_r} \bx + \bx^\T (\bH_{\bS_r \bS_r} -  \bHhat_{\bS_r \bS_r}) \bx \big) \\ 
	&\leq \min_{\|\bx\|_2 = 1} \big(\bx^\T \bHhat_{\bS_r \bS_r} \bx \big) + \bq^\T (\bH_{\bS_r \bS_r} - \bHhat_{\bS_r \bS_r}) \bq \\\
	&= \eval_{\min}(\bHhat_{\bS_r \bS_r}) +  \bq^\T (\bH_{\bS_r \bS_r} - \bHhat_{\bS_r \bS_r}) \bq 
	\end{split} 
	\end{align*}
	\noindent where $\| \bq \|_2 = 1$.
	\begin{align*} 
	\begin{split} 
	\eval_{\min}(\bHhat_{\bS_r \bS_r}) &\geq \eval_{\min}(\bH_{\bS_r \bS_r}) - \bq^\T (\bH_{\bS_r \bS_r} - \bHhat_{\bS_r \bS_r}) \bq \\ 
	&\geq \eval_{\min}(\bH_{\bS_r \bS_r}) - ||| \bH_{\bS_r \bS_r} - \bHhat_{\bS_r \bS_r} |||_{2, 2} \ . 
	\end{split}
	\end{align*}
	\noindent In the above, $||| . |||_{2,2}$ is the spectral norm which is bounded above by the Frobenius norm.
	\begin{align} 
	\label{eq:lambdamin} 
	\begin{split} 
	\eval_{\min}(\bHhat_{\bS_r \bS_r}) &\geq \eval_{\min}(\bH_{\bS_r \bS_r}) - (\sum_{j=1}^{|\bS_r|} \sum_{k=1}^{|\bS_r|} (\bH_{\bS_r \bS_r} - \bHhat_{\bS_r \bS_r})_{jk}^2)^{\frac{1}{2}} \ .\\
	\end{split}
	\end{align}
	\noindent Taking $\epsilon^2 = \frac{\delta^2}{|\bS_r|^2}$ in equation~\eqref{eq:azuma} and using the union bound over
	$|\bS_r|^2$ indexes,
	\begin{align} 
	\label{eq:sample_pop_frobeneus} 
	\begin{split} 
	\prob[||| \bH_{\bS_r \bS_r} - \bHhat_{\bS_r \bS_r} |||_2 &\geq \delta] \leq 2 \exp(- \frac{\delta^2 N }{8 \rho_{\bpi_r \cup \bc_r}^2} +  2 \log \rho_{\bpi_r \cup \bc_r}) \ .  
	\end{split} 
	\end{align}
	Using equations~\eqref{eq:lambdamin} and~\eqref{eq:sample_pop_frobeneus}, it follows that,
	\begin{align} 
	\label{eq:min_eigval_sample} 
	\eval_{\min}(\bHhat_{\bS_r \bS_r}) \geq C - \delta > 0  
	\end{align}
	with probability at least $2 \exp(- \frac{\delta^2 N }{8 \rho_{\bpi_r \cup \bc_r}^2} + 2 \log \rho_{\bpi_r \cup \bc_r}) $.  
\end{proof}

\section{Proof of Lemma \ref{lemma:mutual_incoherence}}
\label{proof:mutual_incoherence}

\begin{proof}
	Using a proof technique similar to \cite{ravikumar2010high}, we can rewrite $\bHhat_{\bS_r^c \bS_r}(\bHhat_{\bS_r \bS_r})^{-1}$ as the sum of four terms defined as:
	\begin{align}
	\label{eq:sample mutual incoherence}
	\begin{split}
	\bHhat_{\bS_r^c \bS_r}(\bHhat_{\bS_r \bS_r})^{-1} &= T_1 + T_2 + T_3 + T4  \\
	|||\bHhat_{\bS_r^c \bS_r}(\bHhat_{\bS_r \bS_r})^{-1}|||_{\tB, \infty, 1} &\leq |||T_1|||_{\tB, \infty, 1} + |||T_2|||_{\tB, \infty, 1} + |||T_3|||_{\tB, \infty, 1} + |||T_4|||_{\tB, \infty, 1}
	\end{split}
	\end{align}
	where,
	\begin{align*}
	T_1 &\triangleq \bH_{\bS_r^c \bS_r}[(\bHhat_{\bS_r \bS_r})^{-1} - \bH_{\bS_r \bS_r}^{-1} ] \\
	T_2 &\triangleq [\bHhat_{\bS_r^c \bS_r} - \bH_{\bS_r^c \bS_r}]\bH_{\bS_r \bS_r}^{-1}\\
	T_3 &\triangleq [\bHhat_{\bS_r^c \bS_r} - \bH_{\bS_r^c \bS_r}][(\bHhat_{\bS_r \bS_r})^{-1} - \bH_{\bS_r \bS_r}^{-1} ] \\
	T_4 &\triangleq \bH_{\bS_r^c \bS_r}\bH_{\bS_r \bS_r}^{-1}
	\end{align*}
	and each $T_i$ is treated as a row partitioned block matrix of $ |( \bpi_r \cup \bc_r)^c|$ blocks with each block containing $\rho_i$ rows where $i \in ( \bpi_r \cup \bc_r)^c$.
	From Mutual incoherence Assumption \ref{assum:mutual_incoherence}, it is clear that $|||T_4|||_{\tB, \infty, 1} \leq 1 - \alpha$. We will control the other three terms by using the following lemma:
	\begin{lemma}
		\label{lem:control_T}
		For any $\delta > 0$, the following holds:
		\begin{align}
		\label{eq:Teq1}
		\begin{split}
		&\prob[|||\bHhat_{\bS_r^c \bS_r} - \bH_{\bS_r^c \bS_r}|||_{\tB, \infty, 1} \geq \delta] \leq 2 \exp(\frac{-\delta^2 N}{8 \bar{\rho}  \rho_{\bpi_r \cup \bc_r}^2} + \log \rho_{(\bpi_r \cup \bc_r)^c} + \log \rho_{\bpi_r \cup \bc_r} )
		\end{split}
		\end{align}
		\begin{align}
		\label{eq:Teq2}
		\begin{split}
		\prob[|||\bHhat_{\bS_r \bS_r} - \bH_{\bS_r \bS_r}|||_{\infty, \infty} \geq \delta] \leq  2 \exp(\frac{-\delta^2 N}{8 \rho_{\bpi_r \cup \bc_r}^2  } + 2 \log \rho_{\bpi_r \cup \bc_r})
		\end{split}
		\end{align}
		\begin{align}
		\label{eq:Teq3}
		\begin{split}
		&\prob[|||(\bHhat_{\bS_r \bS_r})^{-1} - (\bH_{\bS_r \bS_r})^{-1}|||_{\infty, \infty} \geq \delta] \leq 
		2 \exp(- \frac{\delta^2 C^4 N }{32 \rho_{\bpi_r \cup \bc_r}^3} + 2 \log\rho_{\bpi_r \cup \bc_r}
		+ 2 \exp(- \frac{C^2 N }{32 \rho_{\bpi_r \cup \bc_r}^2} + 2 \log \rho_{\bpi_r \cup \bc_r})
		\end{split}
		\end{align}
	\end{lemma}
	\begin{proof}
		Note that,
		\begin{align*}
		\begin{split}
		[\bHhat_{\bS_r^c \bS_r} - \bH_{\bS_r^c \bS_r}]_{jk} &= Z_{jk} \\
		&= \frac{1}{N} \sum_{l=1}^N (\encode(\bX_{\bar r})_j^l \encode(\bX_{\bar r})_k^l - \E_{\mathbb{X}}[\encode(X_{\bar r})_j \encode(X_{\bar r})_k] ) \\
		&= \frac{1}{N}\sum_{l=1}^N Z_{jk}^l
		\end{split}
		\end{align*}
		Let $\Prev(i)$ be the last index before block corresponding to  variable $i \in (\bpi_r \cup \bc_r)^c$ starts. Now,
		\begin{align*}
		||| \bHhat_{\bS_r^c \bS_r} - \bH_{\bS_r^c \bS_r} |||_{\tB, \infty, 1} = \max_{i \in (\bpi_r \cup \bc_r)^c} (\sum_{j = \Prev(i) + 1}^{\Prev(i) + m_i - 1} \sum_{k \in [|\nbr|] } |Z_{jk}|)
		\end{align*}
		Using Hoeffding inequality, we get
		\begin{align*}
		&\prob[|Z_{jk}| \geq \epsilon] \leq 2 \exp(\frac{-\epsilon^2 N}{8}) 
		\end{align*}
		Taking $\epsilon = \frac{\delta}{\rho_i \rho_{\bpi_r \cup \bc_r}} $ for any $i \in (\bpi_r \cup \bc_r)^c$.
		\begin{align*}
		&\prob[|Z_{jk}| \geq \frac{\delta}{\rho_i \rho_{\bpi_r \cup \bc_r}} ] \leq 2 \exp(\frac{-\delta^2 N}{8 (\rho_i \rho_{\bpi_r \cup \bc_r})^2}) \leq 2 \exp(\frac{-\delta^2 N}{8 (\bar{\rho} \rho_{\bpi_r \cup \bc_r} )^2})
		\end{align*}
		Using the union bound over $i \in (\bpi_r \cup \bc_r)^c$ we can write,
		\begin{align*}
		\begin{split}
		&\prob[||| \bHhat_{\bS_r^c \bS_r} - \bH_{\bS_r^c \bS_r} |||_{\tB, \infty, 1} \geq \delta ] \leq \sum_{i\in (\bpi_r \cup \bc_r)^c} \prob[\sum_{j = \Prev(i) + 1}^{\Prev(i) + m_i - 1} \sum_{k \in [|\nbr|] } |Z_{jk}| \geq \delta] \\
		&\leq \sum_{i\in (\bpi_r \cup \bc_r)^c} \prob[\exists j,k | |Z_{jk}| \geq \frac{\delta}{\rho_i  \rho_{\bpi_r \cup \bc_r}}] 
		\\
		&\leq \sum_{i\in (\bpi_r \cup \bc_r)^c} \rho_i  \rho_{\bpi_r \cup \bc_r}  \prob[|Z_{jk}| \geq \frac{\delta}{\rho_i   \rho_{\bpi_r \cup \bc_r}}] \\
		&\leq \sum_{i\in (\bpi_r \cup \bc_r)^c} \rho_i  \rho_{\bpi_r \cup \bc_r}  2 \exp(\frac{-\delta^2 N}{8 (\rho_i  \rho_{\bpi_r \cup \bc_r})^2}) \\
		&\leq 2 \exp(\frac{-\delta^2 N}{8 (\bar{\rho}  \rho_{\bpi_r \cup \bc_r})^2} + \log \rho_{(\bpi_r \cup \bc_r)^c}  + \log \rho_{\bpi_r \cup \bc_r} )
		\end{split}
		\end{align*}
		Similarly we can prove equation~\eqref{eq:Teq2},
		\begin{align*}
		&\prob[||| \bHhat_{\bS_r \bS_r} - \bH_{\bS_r \bS_r} |||_{\infty, \infty} \geq \delta ] 
		\leq  \rho_{\bpi_r \cup \bc_r} \prob[\sum_{k \in \bS_r} |Z_{jk}| \geq \delta  ] \\
		&\leq \rho_{\bpi_r \cup \bc_r}^2 \prob[|Z_{jk}| \geq \frac{\delta}{\rho_{\bpi_r \cup \bc_r}}  ] \\
		&\leq  2 \exp(\frac{-\delta^2 N}{8 \rho_{\bpi_r \cup \bc_r}^2  } + 2 \log \rho_{\bpi_r \cup \bc_r})
		\end{align*}
		Now we will prove equation~\eqref{eq:Teq3}. Note that,
		\begin{align*}
		\begin{split}
		&|||(\bHhat_{\bS_r \bS_r})^{-1} - (\bH_{\bS_r \bS_r})^{-1}|||_{\infty, \infty} = |||(\bH_{\bS_r \bS_r})^{-1} [\bH_{\bS_r \bS_r}
		- \bHhat_{\bS_r \bS_r}] (\bHhat_{\bS_r \bS_r})^{-1}|||_{\infty, \infty} \\
		&\leq \sqrt{|\bS_r|} |||(\bH_{\bS_r \bS_r})^{-1} [\bH_{\bS_r \bS_r} - \bHhat_{\bS_r \bS_r}] (\bHhat_{\bS_r \bS_r})^{-1}|||_{2,2} \\
		&\leq \sqrt{|\bS_r|} |||(\bH_{\bS_r \bS_r})^{-1}|||_{2,2} |||[\bH_{\bS_r \bS_r} - \bHhat_{\bS_r \bS_r}]|||_{2,2} |||(\bHhat_{\bS_r \bS_r})^{-1}|||_{2,2} \\
		&\leq \frac{\sqrt{|\bS_r|}}{C}|||[\bH_{\bS_r \bS_r} - \bHhat_{\bS_r \bS_r}]|||_{2,2} |||(\bHhat_{\bS_r \bS_r})^{-1}|||_{2,2}
		\end{split}
		\end{align*}
		Recall that we proved in equation~\eqref{eq:min_eigval_sample} that  $\prob[\eval_{\min}(\bHhat_{\bS_r \bS_r}) \geq C -
		\delta] \geq 1 - 2 \exp(- \frac{\delta^2 N }{8 \rho_{\bpi_r \cup \bc_r}^2} + 2 \log \rho_{\bpi_r \cup \bc_r}) $. Taking $\delta = \frac{C}{2}$, we get
		$\prob[\eval_{\min}(\bHhat_{\bS_r \bS_r}) \geq \frac{C}{2} ] \geq 1 - 2 \exp(- \frac{C^2 N }{32 \rho_{\bpi_r \cup \bc_r}^2} + 2 \log \rho_{\bpi_r \cup \bc_r}) $.
		This means that,
		\begin{align}
		\label{eq:inverse_eigval}
		\begin{split}
		\prob[|||(\bHhat_{\bS_r \bS_r})^{-1}|||_{2,2} \leq \frac{2}{C} ] &\geq 1 - 2 \exp(- \frac{C^2 N }{32 \rho_{\bpi_r \cup \bc_r}^2} + 2 \log \rho_{\bpi_r \cup \bc_r}) \ .
		\end{split}
		\end{align}
		Furthermore, from equation~\eqref{eq:sample_pop_frobeneus} we have:
		\begin{align*}
		\begin{split}
		\prob[||| \bH_{\bS_r \bS_r} - \bHhat_{\bS_r \bS_r} |||_{2,2} &\geq \epsilon] \leq 2 \exp(- \frac{\epsilon^2 N }{8 \rho_{\bpi_r \cup \bc_r}^2} + 2 \log \rho_{\bpi_r \cup \bc_r}) \ 
		\end{split}
		\end{align*}
		Taking $\epsilon = \delta \frac{C^2}{2\sqrt{|\bS_r|}}$, we get:
		\begin{align*}
		\begin{split}
		\prob[||| \bH_{\bS_r \bS_r} - \bHhat_{\bS_r \bS_r} |||_{2,2} &\geq \delta \frac{C^2}{2\sqrt{|\bS_r|}}] \leq 2 \exp(- \frac{\delta^2 C^4 N }{32\rho_{\bpi_r \cup \bc_r}^3} + 2 \log \rho_{\bpi_r \cup \bc_r}) 
		\end{split}
		\end{align*}
		\noindent It follows that,
		\begin{align*}
		\begin{split}
		&\prob[	|||(\bHhat_{\bS_r \bS_r})^{-1} - (\bH_{\bS_r \bS_r})^{-1}|||_{\infty, \infty} \leq \delta] \geq 1 -  2 \exp(- \frac{\delta^2 C^4 N }{32 \rho_{\bpi_r \cup \bc_r}^3} + 2 \log\rho_{\bpi_r \cup \bc_r} - 2 \exp(- \frac{C^2 N }{32 \rho_{\bpi_r \cup \bc_r}^2} + 2 \log \rho_{\bpi_r \cup \bc_r})
		\end{split}
		\end{align*}
	\end{proof}
	{\bf Controlling the first term of equation~\eqref{eq:sample mutual incoherence}.} 
	
	We can write $T_1$ as,
	\begin{align*}
	&T_1 = - \bH_{\bS_r^c \bS_r} (\bH_{\bS_r \bS_r})^{-1}[\bHhat_{\bS_r \bS_r} - \bH_{\bS_r \bS_r}](\bHhat_{\bS_r \bS_r})^{-1} 
	\end{align*}
	\noindent then,
	\begin{align*}
	\begin{split}
	|||T_1|||_{\tB, \infty, 1} &= ||| \bH_{\bS_r^c \bS_r} (\bH_{\bS_r \bS_r})^{-1}[\bHhat_{\bS_r \bS_r} - \bH_{\bS_r \bS_r}](\bHhat_{\bS_r \bS_r})^{-1} |||_{\tB, \infty, 1} \\
	&\leq ||| \bH_{\bS_r^c \bS_r} (\bH_{\bS_r \bS_r})^{-1} |||_{\tB, \infty, 1} |||[\bHhat_{\bS_r \bS_r} - \bH_{\bS_r \bS_r}]|||_{\infty, \infty} |||(\bHhat_{\bS_r \bS_r})^{-1} |||_{\infty, \infty} \\
	&\leq (1 - \alpha) |||[\bHhat_{\bS_r \bS_r} - \bH_{\bS_r \bS_r}]|||_{\infty, \infty} \sqrt{|\bS_r|}|||(\bHhat_{\bS_r \bS_r})^{-1} |||_{2, 2}
	\end{split}
	\end{align*}
	The first inequality follows using norm inequalities from Section \ref{sec:norm inequalities}.  Now using equation~\eqref{eq:inverse_eigval} and equation~\eqref{eq:Teq2} with $\delta = \frac{\alpha C}{12\sqrt{|\bS_r|}(1 - \alpha)}$ we can say that,
	\begin{align*}
	\begin{split}
	&\prob[|||T_1|||_{\tB, \infty, 1} \leq \frac{\alpha}{6} ] \geq  1 -  2 \exp(- \frac{C^2 N }{32 \rho_{\bpi_r \cup \bc_r}^2} + 2 \log \rho_{\bpi_r \cup \bc_r})   - 2 \exp(-K\frac{N\alpha^2C^2}{144 (1-\alpha)^2\rho_{\bpi_r \cup \bc_r}^3} + 2 \log \rho_{\bpi_r \cup \bc_r})  
	\end{split}
	\end{align*}
	
	{\bf Controlling the second term of equation~\eqref{eq:sample mutual incoherence}.}
	
	We can write $|||T_2|||_{\tB, \infty, 1}$ as,
	\begin{align*}
	\begin{split}
	|||T_2|||_{\tB, \infty, 1} &= |||  [\bHhat_{\bS_r^c \bS_r} - \bH_{\bS_r^c \bS_r}]\bH_{\bS_r \bS_r}^{-1} |||_{\tB, \infty, 1} \\
	&\leq |||  [\bHhat_{\bS_r^c \bS_r} - \bH_{\bS_r^c \bS_r}] |||_{\tB, \infty, 1} |||\bH_{\bS_r \bS_r}^{-1} |||_{\infty, \infty}\\
	&\leq |||  [\bHhat_{\bS_r^c \bS_r} - \bH_{\bS_r^c \bS_r}] |||_{\tB, \infty, 1} \sqrt{|\bS_r|}|||\bH_{\bS_r \bS_r}^{-1} |||_{2,2}\\
	&\leq \frac{\sqrt{|\bS_r|}}{C}|||  [\bHhat_{\bS_r^c \bS_r} - \bH_{\bS_r^c \bS_r}] |||_{\tB, \infty, 1} 
	\end{split}
	\end{align*}
	Again we use norm inequalities from \eqref{sec:norm inequalities} in first inequality. Using equation~\eqref{eq:Teq1} with $\delta = \frac{\alpha C}{6 \sqrt{|\bS_r|}}$ we get,
	\begin{align*}
	\begin{split} 
	&\prob[|||T_2|||_{\tB, \infty, 1} \leq \frac{\alpha}{6} ] \geq  1 - 2 \exp(\frac{-\alpha^2 C^2 N}{48 \bar{\rho}^2  \rho_{\bpi_r \cup \bc_r}^3} +  \log \rho_{(\bpi_r \cup \bc_r)^c} + \log \rho_{\bpi_r \cup \bc_r} )
	\end{split}
	\end{align*}
	
	{\bf Controlling the third term of equation~\eqref{eq:sample mutual incoherence}.} 
	
	We can write $|||T_3|||_{\tB, \infty, 1}$ as, 
	\begin{align*}
	\begin{split}
	&|||T_3|||_{\tB, \infty, 1} \leq 
	|||[\bHhat_{\bS_r^c \bS_r} - \bH_{\bS_r^c \bS_r}] |||_{\tB, \infty, 1} |||[(\bHhat_{\bS_r \bS_r})^{-1} - \bH_{\bS_r \bS_r}^{-1} ]|||_{\infty, \infty} 
	\end{split}
	\end{align*}
	\noindent Using equation~\eqref{eq:Teq1} and~\eqref{eq:Teq3} both with $\delta = \sqrt{\frac{\alpha}{6}}$, we get
	\begin{align*}
	\begin{split}
	&\prob[|||T_3|||_{\tB, \infty, 1} \leq \frac{\alpha}{6} ] \geq  1 -  2 \exp(- \frac{\delta^2 C^4 N }{32 \rho_{\bpi_r \cup \bc_r}^3} + 2 \log\rho_{\bpi_r \cup \bc_r}) - 2 \exp(- \frac{C^2 N }{32 \rho_{\bpi_r \cup \bc_r}^2} + 2 \log \rho_{\bpi_r \cup \bc_r}) \\
	&- 2 \exp(\frac{-\alpha N}{48 (\bar{\rho} \rho_{\bpi_r \cup \bc_r})^2} + \log \rho_{(\bpi_r \cup \bc_r)^c} + \log \rho_{\bpi_r \cup \bc_r} )
	\end{split}
	\end{align*}
	\noindent Putting everything together we get,
	\begin{align*}
	\begin{split}
	&\prob[||| \bHhat_{\bS_r^c \bS_r}\bHhat_{\bS_r \bS_r}^{-1} |||_{\tB, \infty, 1} \leq 1 - \frac{\alpha}{2}] \geq 1 - O( \exp(\frac{-K N}{ \bar{\rho}^2  \rho_{\bpi_r \cup \bc_r}^3} + \log \rho_{(\bpi_r \cup \bc_r)^c} + \log \rho_{\bpi_r \cup \bc_r} ))
	\end{split}
	\end{align*}
	which approaches $1$ as long as we have $N > \bar{\rho}^2  \rho_{\bpi_r \cup \bc_r}^3 \log \rho_{[n]}$
\end{proof}

\section{Discussion on Illustrative Example}
\label{sec:discussion}

For the binary Bayesian network shown in Figure \ref{fig:small bn}, we can explicitly derive expressions for mutual incoherence. We define a symmetric matrix $\bM \in \real^{4\times 4}$ such that $\bM_{ij} = \E[\encode(X_i) \encode(X_j)], \forall i,j \in \{1,2,3,4\}$. Note that $\bM_{ii} = 1, \forall i \in \{1,2,3,4\}$ as each $\encode(X_i) \in \{-1,1\}$. For ease of notation, let $\E[\encode(X_1) \encode(X_2)] = \E[\encode(X_2) \encode(X_4)] = \E[\encode(X_3) \encode(X_4)] = p$ and $\E[\encode(X_1) \encode(X_4)] = q$. Assuming that $\E[\encode(X_1)] = \E[\encode(X_3)] = 0$ and using independence properties of Bayesian networks, we can write $\bM$ as,
\begin{align*}
\bM = \begin{bmatrix}
1 & p & 0 & q \\
p & 1 & 0 & p \\
0 & 0 & 1 & p \\
q & p & p & 1
\end{bmatrix}
\end{align*} 

\subsection{On mutual incoherence (Assumption \ref{assum:mutual_incoherence})}
\label{subsec:discussion on assumption}

We remind the readers that in case of binary variables, Assumption \ref{assum:mutual_incoherence} reduces to $||| \bH_{\bS_r^c\bS_r}\bH_{\bS_r \bS_r}^{-1}|||_{\infty, \infty} < 1 - \alpha$ for some $\alpha \in (0, 1]$. Now, we will derive the necessary conditions to satisfy  Assumption \ref{assum:mutual_incoherence} for each node.

\paragraph{For node $1$: }
We have $\bS_1 = \{2\}$ and $\bS_1^c = \{3,4\}$. Assumption \ref{assum:mutual_incoherence} implies,
\begin{align*}
\begin{split}
||| \begin{bmatrix} \bM_{32} \\ \bM_{42} \end{bmatrix} \bM_{22}^{-1} |||_{\infty, \infty} &< 1 \\
||| \begin{bmatrix} 0 \\ p \end{bmatrix} 1 |||_{\infty, \infty} &< 1 \\
| p | &< 1
\end{split}  
\end{align*}
\paragraph{For node $2$: }
We have $\bS_2 = \{1,4\}$ and $\bS_2^c = \{3\}$. Assumption \ref{assum:mutual_incoherence} implies,
\begin{align*}
\begin{split}
||| \begin{bmatrix} \bM_{31} & \bM_{34} \end{bmatrix} \begin{bmatrix} \bM_{11} & \bM_{14} \\ \bM_{41} & \bM_{44} \end{bmatrix}^{-1} |||_{\infty, \infty} &< 1 \\
||| \begin{bmatrix} 0 & p \end{bmatrix} \begin{bmatrix} 1 & q \\ q & 1 \end{bmatrix}^{-1} |||_{\infty, \infty} &< 1 \\
||| \begin{bmatrix} \frac{-pq}{1 - q^2} & \frac{p}{1-q^2} \end{bmatrix} |||_{\infty, \infty} &< 1 \\
|p| + |q| &< 1
\end{split}  
\end{align*}
\paragraph{For node $3$: }
We have $\bS_3 = \{4\}$ and $\bS_3^c = \{1,2\}$. Assumption \ref{assum:mutual_incoherence} implies,
\begin{align*}
\begin{split}
||| \begin{bmatrix} \bM_{14} \\ \bM_{24} \end{bmatrix} \bM_{44}^{-1} |||_{\infty, \infty} &< 1 \\
||| \begin{bmatrix} q \\ p \end{bmatrix} 1 |||_{\infty, \infty} &< 1 \\
\max (|q|, | p |) &< 1
\end{split} 
\end{align*}
\paragraph{For node $4$: }
We have $\bS_4 = \{2,3\}$ and $\bS_4^c = \{1\}$. Assumption \ref{assum:mutual_incoherence} implies,
\begin{align*}
\begin{split}
||| \begin{bmatrix} \bM_{12} & \bM_{13} \end{bmatrix} \begin{bmatrix} \bM_{22} & \bM_{23} \\ \bM_{32} & \bM_{33} \end{bmatrix}^{-1} |||_{\infty, \infty} &< 1 \\
||| \begin{bmatrix} p & 0 \end{bmatrix} \begin{bmatrix} 1 & 0 \\ 0 & 1 \end{bmatrix}^{-1} |||_{\infty, \infty} &< 1 \\
|p| &< 1
\end{split}  
\end{align*}
Note that conditions for node $1, 3$ and $4$ are already satisfied using our assumptions. Thus we obtain the nontrivial condition that $|\E[\encode(X_1) \encode(X_4)] | + |\E[\encode(X_3) \encode(X_4)]|  < 1$. 

\subsection{On $\bW^*$ in Theorem \ref{theorem:main_result}}
\label{subsec:on wstarmin}

We can compute analytical expression for $\bW_{\nbr.}^*$ for the example binary Bayesian network from Figure \ref{fig:small bn} by using the formula $\bW^*_{\nbr .} = \E[\encode(X_{\bS_r})\encode(X_{\bS_r})^\T]^{-1}\E[\encode(X_{\bS_r})\encode(X_r)^\T]$ and then verify that all its entries are sufficiently away from zero. Note that for the binary variables $\|\f(\bW_i^*)\|_2$ is simply $|\bW_i^*|$.

\paragraph{For node $1$: }
We have $\bS_1 = \{2\}$. 
\begin{align*}
\begin{split}
\bW^*_{\bS_1.} &= \bM_{22}^{-1} \bM_{12} \\
&= 1 \times p \\ 
&= p
\end{split}  
\end{align*}
\paragraph{For node $2$: }
We have $\bS_2 = \{1, 4\}$.
\begin{align*}
\begin{split}
\bW^*_{\bS_2.} &= \begin{bmatrix} \bM_{11} & \bM_{14} \\ \bM_{41} & \bM_{44} \end{bmatrix}^{-1} \begin{bmatrix} \bM_{12} \\ \bM_{42} \end{bmatrix} \\
&= \begin{bmatrix} 1 & q \\ q & 1 \end{bmatrix}^{-1} \begin{bmatrix} p \\ p \end{bmatrix} \\ 
&= \begin{bmatrix} \frac{p}{1 + q} \\ \frac{p}{1 + q} \end{bmatrix}
\end{split}  
\end{align*}
\paragraph{For node $3$: }
We have $\bS_3 = \{ 4\}$. 
\begin{align*}
\begin{split}
\bW^*_{\bS_3.} &= \bM_{44}^{-1} \bM_{43} \\
&= 1 \times p \\ 
&= p
\end{split}  
\end{align*}
\paragraph{For node $4$: }
We have $\bS_4 = \{2,3\}$.
\begin{align*}
\begin{split}
\bW^*_{\bS_4.} &= \begin{bmatrix} \bM_{22} & \bM_{23} \\ \bM_{32} & \bM_{33} \end{bmatrix}^{-1} \begin{bmatrix} \bM_{24} \\ \bM_{34} \end{bmatrix} \\
&= \begin{bmatrix} 1 & 0 \\ 0 & 1 \end{bmatrix}^{-1} \begin{bmatrix} p \\ p \end{bmatrix} \\ 
&= \begin{bmatrix} p \\ p \end{bmatrix}
\end{split}  
\end{align*}

Clearly, none of the $\bW_{\bS_r.}^*$ contains any zero entry. Thus the third statement of Theorem \ref{theorem:main_result} holds as long as  $\min(|p|, |\frac{p}{1 + q}|)$ is sufficiently away from zero.

\section{Norm Inequalities}
\label{sec:norm inequalities}
Here we will derive some norm inequalities which we will use in our proofs.
\begin{lemma}[Norm Inequalities]
	Let $\bA$ be a row partitioned block matrix which consists of $p$ blocks where block $\bA_i \in \real^{m_i \times n}, \ \forall i \in [p]$ and $\bB \in \real^{n \times o}$. Then the following inequalities hold:
	\begin{align*}
	\|\bA \bB \|_{\tB, \infty, 2} \leq  \| \bA\|_{\tB, \infty, 1} \| \bB\|_{\infty,2}
	\end{align*} 
	\begin{align*}
	\|\bA \bB \|_{\tB, \infty, 1} \leq  \| \bA\|_{\tB, \infty, 1} \| \bB\|_{\infty,\infty}
	\end{align*} 
\end{lemma}
\begin{proof}
	Let $\bY$ be a row partitioned block matrix with same size and block structure as $\bA$. Using definitions from Subsection \ref{subsec:notation}:
	\begin{align*}
	\begin{split}
	\|\bA \bB\|_{\tB, \infty, 2} &= \max_{i \in [p]} \|\f((\bA\bB)_i)\|_2 \\
	&= \max_{i \in [p], \| \f(\bY_i) \|_2 \leq 1} \f((\bA\bB)_i)^\T \f(\bY_i) \\
	&= \max_{i \in [p], \| \f(\bY_i) \|_2 \leq 1} [(\bA_i)_{1.}\bB \dots (\bA_i)_{m_i.}\bB] \f(\bY_i) \\
	&= \max_{i \in [p], \| \f(\bY_i) \|_2 \leq 1} [(\bA_i)_{1.}\bB (\bY_i)_{1.} + \dots + (\bA_i)_{m_i.}\bB (\bY_i)_{m_i.}] \\
	&\leq  \max_{\substack{i \in [p],\\ \| \f(\bY_i) \|_2 \leq 1}} \| (\bA_i)_{1.} \|_1 \| \bB (\bY_i)_{1.} \|_{\infty} + \dots +\| (\bA_i)_{m_i.} \|_1 \| \bB (\bY_i)_{m_i.} \|_{\infty} \\
	&\leq \| \bA\|_{\tB, \infty, 1} \| \bB\|_{\infty,2}
	\end{split}
	\end{align*} 
	
	We follow a similar procedure for the last norm inequality.
	\begin{align*}
	\begin{split}
	\|\bA \bB\|_{\tB, \infty, 1} &= \max_{i \in [p]} \|\f((\bA\bB)_i)\|_1 \\
	&= \max_{i \in [p], \| \f(\bY_i) \|_\infty \leq 1} \f((\bA\bB)_i)^\T \f(\bY_i) \\
	&= \max_{i \in [p], \| \f(\bY_i) \|_\infty \leq 1} [(\bA_i)_{1.}\bB \dots (\bA_i)_{m_i.}\bB] \f(\bY_i) \\
	&= \max_{i \in [p], \| \f(\bY_i) \|_\infty \leq 1} [(\bA_i)_{1.}\bB (\bY_i)_{1.} + \dots + (\bA_i)_{m_i.}\bB (\bY_i)_{m_i.}] \\
	&\leq  \max_{i \in [p], \| \f(\bY_i) \|_\infty \leq 1} \| (\bA_i)_{1.} \|_1 \| \bB (\bY_i)_{1.} \|_{\infty} + \dots \| (\bA_i)_{m_i.} \|_1 \| \bB (\bY_i)_{m_i.} \|_{\infty} \\
	&\leq \| \bA\|_{\tB, \infty, 1} \| \bB\|_{\infty,\infty}
	\end{split}
	\end{align*} 	
\end{proof}

\section{Proof of Theorem \ref{theorem:main_result}}
\label{subsec:proof of theorem 1}

In this section, we provide the primal dual construction for the proof of Theorem \ref{theorem:main_result}. Let us consider the block $l_{12}$-norm of $\bW$.
\begin{align}
\label{eq:dual of W}
\begin{split}
\| \bW \|_{\tB, 1, 2}  &=  \sum_{i=1, i\ne r}^n  \|\f(\bW_i) \|_2 \\
&= \sum_{i=1, i\ne r}^n \sup_{\|\f(\bZ_i)\|_2 \leq 1} \f(\bZ_i)^T \f(\bW_i) 
\end{split}
\end{align}
where $\bZ_i$ is a matrix of same size as $\bW_i, \ \forall i \in [n] \wedge i \ne r $. We can think of a row partitioned block matrix $\bZ$ which contains $\bZ_i$ as the row blocks. We can simplify equation~\eqref{eq:dual of W} in following way,
\begin{align}
\label{eq:dual of W 1}
\| \bW \|_{\tB, 1, 2}  = \sum_{i=1, i\ne r}^n \f(\bZ_i)^\T \f(\bW_i) 
\end{align} 
where $\bZ \in \calZ$ and $\calZ$ is defined as follows:
\begin{align}
\label{eq:calZ}
\calZ = \big\{ \bZ \ | \ \f(\bZ_i)  = \begin{cases}
\frac{\f(\bW_i)}{\|\f(\bW_i)\|_2}, \text{when }   \f(\bW_i)  \ne \bzero \\
\f(\bZ_i),  \|\f(\bZ_i)\|_2 \leq 1, \text{otherwise}
\end{cases}, \forall  i \in [n] \wedge i \ne r \big\}
\end{align}
Using equation~\eqref{eq:dual of W 1}, we can rewrite the optimization problem in \eqref{eq:our_model} as follows:

\begin{align}
\label{eq:our model 1}
\begin{split}
\bWhat  = & \min_{\bW} \frac{1}{2N} \| \encode(\bXr) - \encode(\bXnr) \bW \|_F^2 + \lambdahat \sum_{i=1, i\ne r}^n \f(\bZ_i)^\T \f(\bW_i) \\
& \text{such that }  \bZ \in \calZ
\end{split}
\end{align}
At the optimum, the stationarity condition for the optimization problem \eqref{eq:our model 1} is given by:
\begin{align}
\label{eq:stationarity}
\nabla_{\bW}\big[\frac{1}{2N} \| \encode(\bXr) - \encode(\bXnr) \bW\|_F^2 \big] + \lambdahat \bZ = \bzero 
\end{align}
where $\bZ \in \calZ$. 

We use the optimality condition \eqref{eq:stationarity} to prove Theorem \ref{theorem:main_result}. The outline of the proof is as follows:
\begin{enumerate}
	\item First, we fix row blocks of $\bW$ matrix corresponding to non-neighbor nodes of node $r$  to be the zero matrix, i.e., $\bW_i = \bzero, \ \forall i \in (\bpi_r \cup \bc_r)^c$. Then we show that the solution to the optimization problem~\eqref{eq:our_model} is unique. 
	\item We show that $\| \f(\bZ_i)  \|_2 < 1, \forall i \in (\bpi_r \cup \bc_r)^c$ which suffices to justify our choice of $\bW$ in Step 1. 
	\item We prove that $\|\f(\bW_i)\|_2 > 0, \forall i \in \bpi_r \cup \bc_r$ as long as $\|\f(\bW_i^*)\|_2$ is sufficiently large. Requirement on $\|\f(\bW_i^*)\|_2$ is similar to the minimum weight requirement 
\end{enumerate}

We start the proof with the first statement.

\paragraph{Proof of the first statement of Theorem \ref{theorem:main_result}}To prove our first statement, we choose $\bWhat$ such that $\bWhat_{\nbr^c.} = \bzero$. We will show that the optimization problem~\eqref{eq:our_model} has a unique solution for this particular choice of $\bWhat$.
\begin{lemma} 
	\label{lemma:unique_solution}
	Let $\tilde \bW_{\nbr^c.} = \bzero$ for every solution $\tilde \bW$ of the optimization problem \eqref{eq:our_model} then $\bHhat_{\nbr \nbr} \succ 0$ implies that the optimization
	problem~\eqref{eq:our_model} restricted to $( \bW_{\nbr.}; \bzero)$ has a unique solution.
\end{lemma}
(See \eqref{proof:unique_solution} for detailed proof.)

The proof uses convexity of the loss function and the fact that $\bHhat_{\nbr \nbr} \succ 0$ with high probability which is true due to Assumptions \ref{assum:positive_definiteness_assumption} ,\ref{lemma:sample_positive_definiteness}. 

\paragraph{Proof of the second statement of Theorem \ref{theorem:main_result}}To prove that the choice $\bWhat_i = \bzero, \forall i \in
(\bpi_r \cup \bc_r)^c$ is justified, we provide a primal-dual construction. We can rewrite the equation~\eqref{eq:stationarity} as:
\begin{align}
\label{eq:stationarity modified}
\nabla_\bW[\frac{1}{2N} \|\encode(\bXnr) \bW^* + \bE - \encode(\bXnr) \bW\|_F^2 ] + \lambdahat \bZ = \bzero 
\end{align}
where $\bZ \in \calZ$ and $\bE \in \real^{N \times m_r - 1}$ is defined as:
\begin{align*}
\bE = \begin{bmatrix}
\be_1^\T \\ \vdots \\ \be_N^\T
\end{bmatrix}
\end{align*} 
Simplifying equation~\eqref{eq:stationarity modified}, we get:
\begin{align}
\label{eq:stationarity simplified}
\begin{split}
&\nabla_\bW[\frac{1}{2N} \|\encode(\bXnr) (\bW - \bW^*) - \bE \|_F^2 ] + \lambdahat \bZ = \bzero \\
& \frac{1}{N} \encode(\bXnr)^\T \encode(\bXnr) (\bW - \bW^*) - \frac{1}{N}  \encode(\bXnr)^\T \bE + \lambdahat \bZ = \bzero\\
\end{split}
\end{align}
By substituting $\bWhat = (\bWhat_{\nbr.}; \bzero)$ in equation~\eqref{eq:stationarity simplified}  and letting $\bHhat = \frac{1}{N}\encode(\bXnr)^\T \encode(\bXnr)$, we can write the above equation in two parts:
\begin{align}
\label{eq:Zs}
\begin{split}
&\bHhat_{\nbr \nbr} (\bW_{\nbr.} - \bW_{\nbr.}^*) - \frac{1}{N}  \encode(\bXnr)_{\nbr.}^\T \bE + \lambdahat \bZ_{\nbr .} = \bzero \\
&\bW_{\nbr.} - \bW_{\nbr.}^* =  \frac{1}{N} \bHhat_{\nbr \nbr}^{-1}  \encode(\bXnr)_{\nbr.}^\T \bE - \lambdahat \bHhat_{\nbr \nbr}^{-1} \bZ_{\nbr .} 
\end{split}
\end{align}
and
\begin{align}
\label{eq:Zsc}
\bHhat_{\nbr^c \nbr} (\bW_{\nbr.} - \bW_{\nbr.}^*) - \frac{1}{N}  \encode(\bXnr)_{\nbr^c.}^\T \bE + \lambdahat \bZ_{\nbr^c .} = \bzero
\end{align}
Substituting $\bW_{\nbr.} - \bW_{\nbr.}^*$ from equation~\eqref{eq:Zs} to equation~\eqref{eq:Zsc}, we get
\begin{align*}
\lambdahat \bZ_{\nbr^c.} &= - \bHhat_{\nbr^c\nbr} \bHhat_{\nbr \nbr}^{-1}  \frac{1}{N} \encode(\bXnr)_{\nbr.}^\T \bE - \lambdahat \bHhat_{\nbr^c\nbr} \bHhat_{\nbr \nbr}^{-1} \bZ_{\nbr .}  + \frac{1}{N} \encode(\bXnr)_{\nbr^c.}^\T \bE 
\end{align*}
Now we will bound $\lambdahat\| \bZ_{\nbr^c.} \|_{\tB, \infty, 2}$ where blocks have size $\real^{m_i - 1 \times m_r - 1} \ \forall i \in (\bpi_r \cup \bc_r)^c $. 
\begin{align*}
\begin{split}
\lambdahat \| \bZ_{\nbr^c.} \|_{\tB, \infty, 2} &= \|  \bHhat_{\nbr^c\nbr} \bHhat_{\nbr \nbr}^{-1}  \frac{\encode(\bXnr)_{\nbr.}^\T \bE}{N} + \lambdahat \bHhat_{\nbr^c\nbr} \bHhat_{\nbr \nbr}^{-1} \bZ_{\nbr .}  - \frac{ \encode(\bXnr)_{\nbr^c.}^\T \bE}{N} \|_{\tB, \infty, 2} \\
&\leq  \| \bHhat_{\nbr^c\nbr} \bHhat_{\nbr \nbr}^{-1}  \frac{1}{N} \encode(\bXnr)_{\nbr.}^\T \bE \|_{\tB, \infty, 2} + \| \lambdahat \bHhat_{\nbr^c\nbr} \bHhat_{\nbr \nbr}^{-1} \bZ_{\nbr .} \|_{\tB, \infty, 2}  \\
&+ \| \frac{1}{N} \encode(\bXnr)_{\nbr^c.}^\T \bE \|_{\tB, \infty, 2} 
\end{split}
\end{align*}
We can further simplify the above equation by using norm inequalities from Section \ref{sec:norm inequalities}:
\begin{align*}
\begin{split}
\lambdahat \| \bZ_{\nbr^c.} \|_{\tB, \infty, 2} 
&\leq  \| \bHhat_{\nbr^c\nbr} \bHhat_{\nbr \nbr}^{-1} \|_{\tB, \infty, 1} \| \frac{1}{N} \encode(\bXnr)_{\nbr.}^\T \bE \|_{\infty, 2} + \\
& \lambdahat \|  \bHhat_{\nbr^c\nbr} \bHhat_{\nbr \nbr}^{-1} \|_{\tB, \infty, 1} \| \bZ_{\nbr .} \|_{\infty, 2}  + \| \frac{1}{N} \encode(\bXnr)_{\nbr^c.}^\T \bE \|_{\tB, \infty, 2} 
\end{split}
\end{align*}
Using Assumptions \ref{assum:mutual_incoherence} and \ref{lemma:mutual_incoherence}, we can write the above equation as:
\begin{align}
\label{eq:Zsc with mutual incoherence}
\begin{split}
\lambdahat \| \bZ_{\nbr^c.} \|_{\tB, \infty, 2} 
&\leq  (1 - \alpha) \| \frac{1}{N} \encode(\bXnr)_{\nbr.}^\T \bE \|_{\infty, 2} + \lambdahat (1 - \alpha)  + \| \frac{1}{N} \encode(\bXnr)_{\nbr^c.}^\T \bE \|_{\tB, \infty, 2} 
\end{split}
\end{align}
where we used the fact that $\| \bZ_{\nbr .} \|_{\tB, \infty, 2} \leq 1 \implies \| \bZ_{\nbr .} \|_{\infty, 2} \leq 1$. Now it only remains to bound $\| \frac{1}{N} \encode(\bXnr)_{\nbr.}^\T \bE \|_{\infty, 2} $ and $\| \frac{1}{N} \encode(\bXnr)_{\nbr^c.}^\T \bE \|_{\tB, \infty, 2} $ which we do in the next lemma.

\begin{lemma}
	\label{lem:bound some terms}
	If $\lambdahat$ satisfies equation~\eqref{eq:lambda_condition} then the following bounds hold with high probability:
	\begin{align*}
	\| \frac{1}{N} \encode(\bXnr)_{\nbr.}^\T \bE \|_{\infty, 2} \leq \frac{\alpha \lambdahat}{4(1-\alpha)}
	\end{align*}
	\begin{align*}
	\| \frac{1}{N} \encode(\bXnr)_{\nbr^c.}^\T \bE \|_{\tB, \infty, 2}  \leq \frac{\alpha \lambdahat}{4}
	\end{align*}
\end{lemma}
(See Section \ref{sec:bounds some terms} for detailed proof.)

We can now substitute the above result in equation~\eqref{eq:Zsc with mutual incoherence} to get,
\begin{align*}
\lambdahat \| \bZ_{\nbr^c.} \|_{\tB, \infty, 2} \leq \lambdahat (1 - \frac{\alpha}{2}) < \lambdahat 
\end{align*}
The above result along with the definition of $\calZ$ in equation~\eqref{eq:calZ} implies that $\bWhat_{\nbr^c.} = \bzero$.

\paragraph{Proof of the third statement of Theorem \ref{theorem:main_result}:} We recall that $\bWhat$ is a row partitioned block matrix. The node $r$ has an edge with a node $i$ only if $\| \f(\bWhat_i) \|_2 > 0$. The results until now ensures that for each node $r$, we do not recover any edge outside its parents and children. Now we prove the third statement of Theorem \ref{theorem:main_result} which makes sure that we recover all the parents and children. We will prove this by using the following lemma.

\begin{lemma}
	\label{lemma:l2_norm_bound}
	If $\lambdahat$  satisfies equation~\eqref{eq:lambda_condition} then,
	\begin{align*}
	\| \bWhat_{\nbr.} - \bW_{\nbr.}^* \|_{\tB, \infty, 2}  \leq  \frac{2\bar{m}}{C} (   \frac{\alpha}{4 (1 - \alpha)} + \sqrt{m_r - 1} +  1 ) \sqrt{|\nbr|} \lambdahat
	\end{align*}
\end{lemma} 
(See \eqref{proof:l2_norm_bound} for detailed proof.)

It follows that if $\min_{i \in \bpi_r \cup \bc_r} \| \f(\bW_i^*) \|_2 > \frac{4 \bar{m}}{C} (   \frac{\alpha}{4 (1 - \alpha)} + \sqrt{m_r - 1} +  1 ) \sqrt{|\nbr|} \lambdahat $ then $\| \f(\bW_i^*) \|_2 > 0$ implies that $\| \f(\bWhat_i) \|_2 > 0$. This in turn implies that we recover the correct set of the parents and children.  

\section{Proof of Lemma \ref{lemma:unique_solution}}
\label{proof:unique_solution}
\begin{proof}
	If we take $\tilde \bW_{{\nbr}^c.}  = \bzero$ then equation~\eqref{eq:our_model} restricted to $\tilde \bW = (\tilde \bW_{\nbr}; \bzero)$ becomes,
	\begin{align*}
	(\tilde\bW_{\nbr.}, \bzero) = \min_{(\bW_{\nbr.}, \bzero)} \hat f((\bw_{\nbr}, \bzero)) \ .
	\end{align*}
	This can be equivalently written as,
	\begin{align}
	\label{eq:restrictedS}
	\tilde\bW_{\nbr.} = \min_{\bw_{\nbr.}} \frac{1}{2N} \| \encode(\bXnr_{.\nbr}) \bW_{\nbr.} - \encode(\bXr) \|_F^2  + \lambdahat  \| \bW_{\nbr.} \|_{\tB, 1, 2}
	\end{align}
	The Hessian of objective function in optimization problem~\eqref{eq:restrictedS} is given as follows:
	
	\begin{align*}
	\nabla^2 \hat\Loss_{\bW_{\nbr.}}(\bW_{\nbr.}) &= \begin{bmatrix}
	\bHhat_{\nbr\nbr} & \bzero & \dots & \bzero \\
	\bzero & \bHhat_{\nbr\nbr} & \dots & \bzero \\
	\vdots & \vdots & \dots & \vdots \\
	\bzero & \bzero & \dots & \bHhat_{\nbr\nbr} 
	\end{bmatrix} 
	\end{align*}
	where  $\bHhat_{\nbr\nbr}  = \frac{1}{N} {\encode(\bXnr)_{\nbr.}}^\T \encode(\bXnr)_{.\nbr} \in \real^{\rho_{\nbr} \times \rho_{\nbr}}$ and  $\nabla^2 \hat\Loss_{\bW_{\nbr.}}(\bW_{\nbr.}) \in \real^{\rho_r\rho_{\nbr} \times \rho_r\rho_{\nbr}}$. Using Assumption \ref{assum:positive_definiteness_assumption} and \ref{lemma:sample_positive_definiteness}, we know that $\bHhat_{\nbr\nbr} \succ 0$ and it follows~\cite{de2006aspects} that $\nabla^2 \hat\Loss_{\bW_{\nbr.}}(\bW_{\nbr.}) \succ 0$. The objective function in optimization problem~\eqref{eq:restrictedS} is strictly convex with respect to $\bW_{\nbr.}$. Thus, it has a unique solution.
\end{proof}

\section{Proof of Lemma \ref{lem:bound some terms}}
\label{sec:bounds some terms}
\begin{proof}
	We start by bounding the first term.
	\paragraph{Bounding $\| \frac{\encode(\bXnr)_{\nbr}^\T \bE}{N}\|_{\infty, 2}$} 
	\begin{align*}
	\| \frac{\encode(\bXnr)_{\nbr}^\T \bE}{N}\|_{\infty, 2} &= \max_{i \in \nbr} \| \frac{1}{N} ( \encode(\bXnr)_{i.}^\T \bE_{.1} \dots \encode(\bXnr)_{i.}^\T \bE_{.m_r - 1}) \|_2 \\
	&\leq \max_{i \in \nbr} \sqrt{m_r - 1} \max_{j \in [m_r - 1]} \frac{1}{N} |\encode(\bXnr)_{i.}^\T \bE_{.j}|  		
	\end{align*}
	We will take a closer look at $\frac{1}{N} |\encode(\bXnr)_{i.}^\T \bE_{.j}|$. 
	\begin{align*}
	\frac{1}{N} |\encode(\bXnr)_{i.}^\T \bE_{.j}| \leq \frac{1}{N} \sum_{k=1}^N |\bE_{jk}|
	\end{align*}
	Note that $|\bE_{jk}|$ is a bounded random variable and hence we can use Hoeffding's inequality,
	\begin{align*}
	\prob[\frac{1}{N} \sum_{k=1}^N |\bE_{jk}| > \mu + t ] \leq \exp(\frac{-Nt^2}{2\sigma^2})
	\end{align*}
	where $\mu$ and $\sigma$ are defined in equation~\eqref{eq:mu 1} and \eqref{eq:sigma 1} respectively. Taking union bound across $i \in \nbr$ and $j \in [m_r-1]$, we get
	\begin{align}
	\label{eq:error norm with S}
	\begin{split}
	\prob[\| \frac{\encode(\bXnr)^\T \bE}{N}\|_{\infty, 2} > \sqrt{\rho_r}(\mu + t)] &\leq \exp\big(\frac{-Nt^2}{2\sigma^2} + \log \rho_{\bpi_r \cup \bc_r}  + \log \rho_r  \big) 
	\end{split}
	\end{align}
	Taking $t = \frac{\lambdahat \alpha}{4\sqrt{\rho_r} (1 - \alpha)} - \mu$, we get
	\begin{align*}
	\begin{split}
	\prob[\| \frac{\encode(\bXnr)_{\nbr}^\T \bE}{N}\|_{\infty, 2} > \frac{\lambdahat \alpha}{4 (1 - \alpha)}] &\leq \exp(\frac{-N (\frac{\lambdahat \alpha}{4\sqrt{\rho_r} (1 - \alpha)} - \mu)^2 }{2\sigma^2} + \log  \rho_{\bpi_r \cup \bc_r} + \log \rho_r  ) 
	\end{split}
	\end{align*}
	
	Now we bound the second term.
	\paragraph{ Bounding $\| \frac{1}{N} \encode(\bXnr)_{\nbr^c}^\T \bE \|_{\tB, \infty, 2}$}
	We denote each row block of $\encode(\bXnr)_{\nbr^c}^\T$ as $\bM_i \in \real^{m_i - 1 \times N}, \forall i \in (\bpi_r \cup \bc_r)^c $.
	\begin{align*}
	\| \frac{1}{N} \encode(\bXnr)_{\nbr^c}^\T \bE \|_{\tB, \infty, 2} &= \max_{i \in (\bpi_r \cup \bc_r)^c} \| \frac{1}{N} [(\bM_i)_{1.} \bE \dots (\bM_i)_{m_i - 1.} \bE] \|_2 \\
	&\leq \max_{i \in (\bpi_r \cup \bc_r)^c} \sqrt{(m_r - 1) (m_i - 1)} \max_{k \in [m_i-1], l \in [m_r - 1] } \frac{|(\bM_i)_{k.} \bE_{.l}|}{N}
	\end{align*}
	We use a similar argument as before keeping in mind that $\bar m = \max_{i \in [n]} m_i$.
	\begin{align*}
	\begin{split} 
	&\prob[\| \frac{1}{N} \encode(\bXnr)_{\nbr^c}^\T \bE \|_{\tB, \infty, 2} > \sqrt{(\rho_r) \bar{\rho}}(\mu + t)] \leq \exp(\frac{-Nt^2}{2\sigma^2} + \log |(\bpi_r \cup \bc_r)^c| + \log (\rho_r \bar{\rho})) 
	\end{split} 
	\end{align*}
	Taking $t = \frac{\lambdahat \alpha}{4\sqrt{\rho_r\bar{\rho}} } - \mu$, we get
	\begin{align*}
	\begin{split} 
	&\prob[\| \frac{\encode(\bX_{\bar r})_{\nbr^c}^\T \bE}{N}\|_{\tB, \infty, 2} > \frac{\lambdahat \alpha}{4}] \leq \exp\big(\frac{-N(\frac{\lambdahat \alpha}{4\sqrt{\rho_r\bar{\rho}} } - \mu)^2}{2\sigma^2} + \log |(\bpi_r \cup \bc_r)^c|  + \log (\rho_r \bar{\rho}) \big)
	\end{split} 
	\end{align*}
	By choosing $\lambdahat$ which satisfy equation \eqref{eq:lambda_condition}, we prove the lemma.
\end{proof}

\section{Proof of Lemma \ref{lemma:l2_norm_bound}}
\label{proof:l2_norm_bound}
\begin{proof}
	First, we get the expression of the difference between $\bWhat_{\nbr.}$ and $\bW^*_{\nbr.}$ from equation~\eqref{eq:Zs}. 
	\begin{align*}
	\begin{split}
	\bWhat_{\nbr.} - \bW_{\nbr.}^* &= \bHhat_{\nbr\nbr}^{-1}\big[ \frac{\encode(\bXnr)_{\nbr}^\T \bE }{N} - \lambdahat \bZ_{\nbr.}\big] \\
	\| \bWhat_{\nbr.} - \bW_{\nbr.}^* \|_{\tB, \infty, 2} &= \| \bHhat_{\nbr\nbr}^{-1}\big[ \frac{\encode(\bXnr)_{\nbr}^\T \bE }{N} - \lambdahat \bZ_{\nbr.}\big] \|_{\tB, \infty, 2} \\ 
	&\leq \| \bHhat_{\nbr\nbr}^{-1}\big[ \frac{\encode(\bXnr)_{\nbr}^\T \bE }{N} \|_{\tB, \infty, 2} + \| \lambdahat \bZ_{\nbr.}\big] \|_{\tB, \infty, 2}\\
	&\leq \| \bHhat_{\nbr\nbr}^{-1} \|_{\tB, \infty, 1} \| \frac{\encode(\bXnr)_{\nbr}^\T \bE }{N} \|_{\infty, 2} + \lambdahat \| \bHhat_{\nbr\nbr}^{-1} \|_{\tB, \infty, 1} \|  \bZ_{\nbr} \|_{\infty, 2} \\
	&\leq \bar{m} \| \bHhat_{\nbr\nbr}^{-1} \|_{\infty, \infty} \| \frac{\encode(\bXnr)_{\nbr}^\T \bE }{N} \|_{\infty, 2} + \lambdahat \bar{m} \| \bHhat_{\nbr\nbr}^{-1} \|_{\infty, \infty}  \\
	&\leq \bar{m} \sqrt{|\nbr|} \| \bHhat_{\nbr\nbr}^{-1} \|_{2,2} \| \frac{\encode(\bXnr)_{\nbr}^\T \bE }{N} \|_{\infty, 2} + \lambdahat  \bar{m}  \sqrt{|\nbr|} \| \bHhat_{\nbr\nbr}^{-1} \|_{2,2} \\
	&\leq \bar{m} \sqrt{|\nbr|} \frac{2}{C} \sqrt{m_r - 1} (\mu + t) + \lambdahat  \bar{m}  \sqrt{|\nbr|} \frac{2}{C} 
	\end{split}
	\end{align*}
	The second inequality comes from norm inequalities discussed in Section \ref{sec:norm inequalities}. The third inequality follows because by the definitions of both the norms. The fourth and fifth inequalities use the bounds from equations~\eqref{eq:error norm with S} and \eqref{eq:inverse_eigval} which hold with high probability if $\lambdahat$ satisfies equation~\eqref{eq:lambda_condition}.Taking $t = \lambda$ and noting that $\lambda >  \frac{4 \sqrt{m_r - 1}(1 - \alpha)}{\alpha} \mu$, we have
	\begin{align*}
	\| \bWhat_{\nbr.} - \bW_{\nbr.}^* \|_{\tB, \infty, 2} \leq \frac{2\bar{m}}{C} (   \frac{\alpha}{4 (1 - \alpha)} + \sqrt{m_r - 1} +  1 ) \sqrt{|\nbr|} \lambdahat
	\end{align*}
\end{proof}

\end{document}